\documentclass[11pt]{article}
\usepackage{ifthen}
\usepackage{mdwlist}
\usepackage{amsmath,amssymb,amsfonts,amsthm}
\usepackage{bm}
\usepackage[colorlinks=true,linkcolor=blue,citecolor=blue,urlcolor=blue]{hyperref}
\usepackage{enumitem}
\usepackage{graphicx}
\usepackage{xspace}
\usepackage{verbatim}
\usepackage{algorithm}
\usepackage{algpseudocode}
\usepackage[margin=1in]{geometry}
\usepackage{color}
\usepackage{thm-restate}
\usepackage{latexsym}
\usepackage{epsfig}
\usepackage[colorinlistoftodos,textsize=scriptsize]{todonotes}
\def\withcolors{1}
\def\withnotes{1}
\def\eps{\ve}
\renewcommand{\epsilon}{\ve}
\def\ve{\varepsilon}

\newcommand{\E}{\mbox{\bf E}}

\newcommand{\pr}[2][]{\mathrm{Pr}\ifthenelse{\not\equal{}{#1}}{_{#1}}{}\!\left[#2\right]}
\newcommand{\R}{\mathbb{R}}

\newtheorem{theorem}{Theorem}

\newtheorem{proposition}[theorem]{Proposition}

\newtheorem{remark}[theorem]{Remark}
\newtheorem{fact}[theorem]{Fact}
\newtheorem{lemma}[theorem]{Lemma}

\newtheorem{corollary}[theorem]{Corollary}

\newtheorem{definition}[theorem]{Definition}

\newtheorem{conjecture}[theorem]{Conjecture}
\newtheorem{construction}[theorem]{Construction}
\numberwithin{theorem}{section} 
\numberwithin{nontheorem}{section} 
\numberwithin{proposition}{section} 
\numberwithin{observation}{section} 
\numberwithin{remark}{section} 
\numberwithin{fact}{section} 
\numberwithin{lemma}{section} 
\numberwithin{claim}{section} 
\numberwithin{corollary}{section} 
\numberwithin{case}{section} 
\numberwithin{dfn}{section} 
\numberwithin{definition}{section} 
\numberwithin{question}{section} 
\numberwithin{openquestion}{section} 
\numberwithin{res}{section} 
\numberwithin{conjecture}{section}
\numberwithin{construction}{section}

\ifnum\withcolors=1
  \newcommand{\gcolor}[1]{{\color{red}#1}} 
\else

  \newcommand{\gcolor}[1]{{#1}}
\fi

\ifnum\withnotes=1
  \newcommand{\gnote}[1]{\par\gcolor{\textbf{G: }\sf #1}} 
  \newcommand{\gfootnote}[1]{\footnote{{\bf \gcolor{Gautam}}: {#1}}}

	\newcommand{\questionnote}[2][]{\todo[size=\tiny,color=blue!30]{#2}}

\else
  \newcommand{\gnote}[1]{}
  \newcommand{\gfootnote}[1]{}
  
	\newcommand{\questionnote}[2][]{\ignore{#2}}

\fi
\newcommand{\ignore}[1]{\leavevmode\unskip} 
\newcommand{\gmargin}[1]{\questionnote{\gcolor{#1}}}

\usepackage{amsmath,amssymb,amsthm,dsfont,mathtools,microtype}
\usepackage[normalem]{ulem}

\newcommand{\ha}[1]{{\color{blue}{Hassan: #1}}}
\newcommand{\pl}[1]{{\color{pink}{Peihan: #1}}}
\newcommand{\vs}[1]{{\color{red}{Vikrant: #1}}}
\newcommand{\am}[1]{{\color{orange}{Argyris: #1}}}
\newcommand{\mo}[1]{{\color{magenta}{Mohammad: #1}}}
\newcommand{\gk}[1]{{\color{green}{Gautam: #1}}}

\newif\ifshowcomments
\showcommentstrue

\ifshowcomments\else
  \renewcommand{\ha}[1]{}
  \renewcommand{\pl}[1]{}
  \renewcommand{\vs}[1]{}
  \renewcommand{\gmargin}[1]{}
  \renewcommand{\todo}[1]{}
  \renewcommand{\gnote}[1]{}
  \renewcommand{\am}[1]{}
\fi

\let\P\relax
\DeclareMathOperator\P{\mathcal{P}} 
\DeclareMathOperator\Q{\mathcal{Q}} 
\DeclareMathOperator\X{\mathcal{X}} 
\DeclareMathOperator\A{\mathcal{A}} 
\DeclareMathOperator\TV{\mathrm{TV}} 
\DeclareMathOperator\q{\mathbf{q}} 
\newcommand\F{\mathcal{F}} 
\DeclareMathOperator\Y{\mathcal{Y}} 

\title{Actually, All Learnable Distribution Classes Are Privately Learnable\thanks{Authors are listed in alphabetical order.}}

\author {
Mohammad Afzali\thanks{University of British Columbia. {\tt moafzali@cs.ubc.ca}. Supported in part by a Four Year Doctoral Fellowship (4YF) from the University of British Columbia.} \and
Hassan Ashtiani\thanks{McMaster University and Vector Institute. {\tt zokaeiam@mcmaster.ca}. Supported by an NSERC Discovery Grant.} \and
Alex Bie\thanks{Google Research. {\tt alexbie@google.com}.} \and
Gautam Kamath\thanks{New York University. {\tt g@csail.mit.edu}. Supported by an NSERC Discovery Grant and an Ontario Early Researcher Award.} \and
Peihan Liu\thanks{Columbia University. {\tt peihanliu@cs.columbia.edu}. Supported in part by the Center for Smart Streetscapes, an NSF Engineering Research Center, under grant agreement EEC-2133516} \and
Argyris Mouzakis\thanks{University of Texas at Austin. Supported by the NSF AI Institute for Foundations of Machine Learning (IFML).} \and
Vikrant Singhal\thanks{University of Copenhagen. {\tt vikrant.singhal@di.ku.dk}. Supported by a Startup grant from Novo Nordisk Foundation (NNF24OC0087820).}
}
\date{}

\begin{document}
\maketitle

\begin{abstract}
We show that every class of distributions that is learnable up to arbitrarily small target error in total variation distance with finitely many samples is also learnable under $(\varepsilon, \delta)$-differential privacy up to arbitrarily small target error with finitely many samples. This proves a conjecture of Ashtiani. 
Previous work by Bun, Kamath, Mouzakis, and Singhal~\cite{BunKMS24} gave a \emph{weak} refutation of this conjecture, showing that this implication is not true when learnability is defined with respect to a fixed target error.
In contrast, we show that, under the standard definition of learnability, the conjecture is true. 
\end{abstract}

\section{Introduction}

Suppose a learner receives samples from an unknown probability distribution $p$, which is known to belong to some class of probability distributions $\mathcal{Q}$. 
Can the learner identify a distribution $\hat p$, such that $\hat p$ is close to the distribution $p$ in total variation distance? 
This fundamental question is known as \emph{distribution learning} or \emph{density estimation}, and has seen significant study across a variety of adjacent research communities including computer science, statistics, and information theory (see, for example, ~\cite{KearnsMRRSS94, DevroyeL01, DaskalakisDS12b, Diakonikolas16,  AshtianiBHLMP20,lechner2026no}). 

The problem of density estimation has also been studied extensively under the constraint of differential privacy~(DP)~\cite{DworkMNS06}. Indeed, there is a rich body of literature on private learning of various classes of distributions. Some examples are private density estimation of structured univariate distributions \cite{DiakonikolasHS15, KarwaV18}, high-dimensional Gaussians \cite{KamathLSU19, BunKSW19,BiswasDKU20,AdenAliAK21, AshtianiL22, KothariMV22, KamathMSSU22, AlabiKTVZ23, HopkinsKMN23}, Gaussian Mixture Models \cite{KamathSSU19,AdenAliAL21,BenDavidBCKS23,AshtianiAL24,ashtiani2024sample,AshtianiAL25}, and product distributions \cite{KamathLSU19,BunKSW19,Singhal24}.

One vibrant program within this area has focused on designing increasingly general and generic frameworks for learning distribution classes (see, for example, ~\cite{BunKSW19, AshtianiL22, HopkinsKMN23, AsiUZ23, AshtianiAL24, AshtianiAL25}).
At their core, many of these methods rely upon non-private learning algorithms as a black box, applying DP techniques on their outputs to design private algorithms. At a July 2022 workshop at the Fields Institute, Ashtiani conjectured that \emph{every} learnable distribution class could be privately learned~\cite{Ashtiani22}.

\begin{conjecture}[Ashtiani's Conjecture]
    Every distribution class $\mathcal{Q}$ that is learnable with respect to total variation distance is also learnable under the additional constraint of $(\varepsilon,\delta)$-differential privacy.
\end{conjecture}

Our main result (Theorem~\ref{thm:main-informal}) is a positive resolution to this conjecture: given a non-private algorithm for learning a distribution class up to arbitrarily small target error in total variation distance, we provide a private algorithm for learning the same distribution class.

A recent work by Bun, Kamath, Mouzakis, and Singhal~\cite{BunKMS24} claimed to refute Ashtiani's conjecture, which may seem to contradict our result.
In fact, their work can be viewed as a weak refutation of Ashtiani's conjecture.
They constructed a distribution class which can be learned up to constant error in total variation distance, but cannot be learned privately to that same error.
The distribution class in their counter-example, however, is not (non-privately) learnable to all values of error. 
In contrast, we use the traditional definition of learnability, wherein the distribution class must be learnable to \emph{arbitrarily small} error \cite{KearnsSS94}. Indeed, our private algorithm relies crucially upon non-private learning of the class to a higher accuracy than the target.

Our positive resolution of the conjecture may come as a surprise as private and non-private learnability are not necessarily equivalent for problems beyond density estimation. Indeed, in the standard binary classification setting, non-private learnability does not imply $(\varepsilon,\delta)$-DP learnability \cite{BunNSV15,alon2019private}. It is also remarkable that for (non-private) density estimation, a nice characterization of learnability is not known (and in some sense impossible~\cite{LechnerB23,lechner2026no}) while our private to non-private reduction works anyway.

\subsection{Results}

In what follows, let $\mathcal X$ be the domain of examples, which we assume is a standard Borel space, and denote by $\Delta(\mathcal X)$ the set of all probability distributions on $\mathcal X$. Any $\mathcal Q \subseteq \Delta(\mathcal X)$ is referred to as a \emph{distribution class}. Denote by TV the total variation distance: for $p,q \in \Delta(\mathcal X)$, $\TV(p,q) \coloneqq \sup_{A \in \mathcal B} |p(A) - q(A)|$  where $\mathcal B$ are the measurable sets of $\mathcal X$.

\begin{definition}[Distribution learning]\label{def:learnability}A distribution class $\mathcal{Q} \subseteq \Delta(\mathcal{X})$ is \textbf{learnable} if there exist a learner $A: \mathcal{X}^* \to \Delta(\mathcal{X})$ and a function $m: (0,1)^2 \to \mathbb{N}$ such that, for every $\alpha,\beta \in (0,1)$, every $n\geq m(\alpha,\beta)$, and every $p\in \mathcal{Q}$, $$\Pr_{S\sim p^n}[\TV(A(S), p)\geq \alpha] \leq \beta.$$ 

The \textbf{sample complexity} of learning $\mathcal{Q}$
is the pointwise infimum of all functions $m$
for which there exists a learner $A$ satisfying the above guarantee.
\end{definition}

\begin{theorem}[Learnable classes are privately learnable; informal version of Theorem~\ref{thm:robust-learner}] \label{thm:main-informal} 
Suppose $\mathcal Q\subseteq\Delta(\mathcal X)$ is learnable with sample complexity $m(\alpha,\beta)$ for $\alpha,\beta \in (0,1)$. Then, for every $\varepsilon,\delta\in(0,1)$, $\mathcal Q$ is also learnable under $(\varepsilon,\delta)$-differential privacy, with sample complexity $n(\alpha,\beta,\varepsilon,\delta)=O\left( m\left( \frac{1}{512m\left(\alpha/3,1/32\right)}, \frac{1}{16} \right) \frac{1}{\varepsilon} \log\left(\frac{4}{\delta}\right) \log\left(\frac{4}{\beta}\right) \right)$.  
\end{theorem}

In particular, if a class $\mathcal{Q}$ is learnable with sample complexity $m(\alpha,\beta) = \text{poly}(\tfrac 1 \alpha, \log \tfrac 1 \beta)$, our bound implies that $\mathcal{Q}$ is $(\varepsilon,\delta)$-DP learnable with sample complexity  $m(\alpha,\beta,\epsilon,\delta) = \text{poly}(\tfrac 1 \alpha, \log \tfrac 1 \beta,\frac 1 \varepsilon, \log \tfrac 1 \delta)$.

As part of our approach, we develop private samplers~\cite{RaskhodnikovaSSS21,CheuN25} for any learnable class, which can be of independent interest. These samplers receive private data points and output a synthetic data set, generated from a distribution that is close to the underlying (unknown) distribution.

\begin{definition}[Approximate multi-sampler]\label{def:sampler}
Let $\Q \subseteq \Delta(\mathcal X)$, $\ell \in \mathbb N$, $\gamma \in (0,1)$. A randomized algorithm
$M : \mathcal X^n \to \mathcal X^\ell$ is a \textbf{$\boldsymbol{(\ell,\gamma)}$-sampler} for $\Q$ with sample complexity $n$ if for every $p\in \Q$, $$\TV(\mathcal{L}(Z), \mathcal{L}(Z')) \leq \gamma,$$
where $Z\sim p^\ell$, $Z' = M(S)$ with $S\sim p^n$, and $\mathcal{L}(\cdot)$ corresponds to the law of the random variable.
\end{definition}

\begin{theorem}[Differentially private synthetic data for learnable classes; informal version of Theorem~\ref{thm:informal-sampler}]\label{thm:sampling-informal}
Let $\Q \subseteq \Delta(\mathcal X)$ be learnable with
sample complexity $m(\alpha, \beta)$. For every $\ell \in \mathbb N$ and $\gamma$, $\varepsilon, \delta \in (0,1)$, there exists a $(\varepsilon,\delta)$-differentially private $(\ell,\gamma)$-sampler for $\Q$ with sample complexity
$$n = O\left( m\left(\frac{\gamma}{64\ell}, \frac 1{16}\right)
\log\left(\frac{1}{\gamma\delta}\right) \frac 1 \varepsilon \right).$$
\end{theorem}

For the problem of (non-private) density estimation, learnability in realizable setting is not equivalent to learnability in agnostic/robust setting \cite{BenDavidBKL23}. Our next result is a private to non-private reduction for agnostic learning of distributions. The formal version and more discussions appear in Section~\ref{sec:agnostic-private}.

\begin{theorem}[Agnostic learnability implies private agnostic learnability, Informal version of~\ref{thm:agnostic-private-general}] \label{thm:agnostic-main-informal} 
Suppose $\mathcal Q\subseteq\Delta(\mathcal X)$ is ``agnostic'' learnable with sample complexity $m(\alpha,\beta)$ for $\alpha,\beta \in (0,1)$. Then, for every $\varepsilon,\delta\in(0,1)$, $\mathcal Q$ is also agnostic learnable under $(\varepsilon,\delta)$-differential privacy, with sample complexity $n(\alpha,\beta,\varepsilon,\delta)=O\left( m\left( \Theta(\alpha), \frac{1}{16} \right) m\left( \Theta(\alpha), \beta \right) \frac{1}{\varepsilon} \log\left(\Theta(\frac{1}{\alpha\delta}\right) \right)$.
\end{theorem}

\subsection{Techniques}\label{sec:techniques}

At its core, our algorithm falls into the subsample-and-aggregate framework~\cite{NissimRS07}.
Suppose we have a non-private learning algorithm $A$, which learns the distribution class $\mathcal{Q}$ in total variation distance.
Consider the following intuitive yet naive approach:
\begin{enumerate}
    \item Partition the dataset into $k=O(\log(1/\delta)/\varepsilon)$ parts;
    \item Run the non-private learning algorithm $A$ on each part;
    \item Privately aggregate the outputs.
\end{enumerate}
Let $q_i$ for $i \in [k]$ be the output of the non-private learning algorithm on the $i^{\text{th}}$ part of the dataset. 
If many of $q_i$'s are exactly the same (i.e., the learner is \emph{globally stable}~\cite{BunLM20}), then they can be easily aggregated, for example, using DP stability-based mode~\cite{Vadhan17}.  
However, the non-private learner $A$ only guarantees that, with high probability, most of $q_i$'s are \emph{close to} the ground truth $p$ in TV, thus to one another. In other words, non-private density estimators may not be globally stable. 

Another idea for aggregating $q_i$'s is to first generate samples from each $q_i$, and then privately aggregate these samples. Assume $Z_i\sim q_i^{\ell}$ is a sample of size $\ell$ from $q_i$. If we are lucky and many of these $Z_i$'s ``collide'' (take the exact same value for different choices of $i$), then again we can privately aggregate them and will end up having a private sample set (synthetic data) from a distribution that is close to the ground truth distribution $p$. These public samples can then be fed to the non-private learner to find an approximation to $p$. 
In first glance, introducing the additional step of sampling from $q_i$'s appears counterproductive. In fact, now $Z_i$'s may not collide not only because $q_i$'s are not colliding, but also because of the randomness of generating $Z_i$'s. Indeed, this would be the case if $Z_i$'s are generated independently for each $i$. However, as we discuss below, allowing $Z_i$'s to be correlated (across values of $i$) gives an extra flexibility that proves to be crucial for stablizing the outcomes of non-private learners.

To make the idea concrete we turn to \emph{couplings}. Couplings construct a joint distribution over a collection of random variables such that their marginal distributions are preserved. The key observation is that a suitable TV coupling can convert closeness in total variation distance into a high probability guarantee of \emph{exact equality} between sampled outcomes. 
Specifically, after obtaining $q_1,...,q_k$ non privately, 
we jointly draw $(Z_1,...,Z_k)$ from a TV coupling of these distributions
satisfying Definition~\ref{def:grand-scheme}. The coupling is designed to satisfy two key desiderata: (1) each marginal has the correct law, so that $Z_i \sim q_i^\ell$ is a valid synthetic dataset from the learner output on the $i^{th}$ block, and (2) assuming $q_i$'s are all close in total variation distance, synthetic datasets $Z_i$'s would take the \emph{exact same value} with high probability. 
Thus, a TV coupling can convert total variation closeness of the $q_i$'s into the sufficiently many exact repetitions of the $Z_i$'s needed by the stability-based mode.

Indeed, the use of correlated sampling to convert closeness in TV into agreement of sampled outcomes already appears in the user-level private learning algorithms of Ghazi, Kumar, and Manurangsi~\cite{GhaziKM21}. Kalavasis, Karbasi, Moran, and Velegkas~\cite{KalavasisKMV23} subsequently use a related coupling approach to reduce TV-indistinguishable PAC learning to approximate differentially private learning. 
A similar concept has been used in
\cite{cohen2026hot} for private aggregation of teacher models.
In our setting, we apply this principle to synthetic datasets.

Our constructions of a couplings are based on the Poisson coupling of Angel and Spinka \cite{AngelS19}, which ensures that for all $i,j \in [k]$ simultaneously, $\Pr[Z_i \neq Z_j] \leq 2 \TV(q_i^\ell, q_j^\ell).$
When all relevant distributions admit densities with respect to a common $\sigma$-finite reference measure, the Poisson coupling of Angel and Spinka~\cite{AngelS19} applies directly, as in~\cite{KalavasisKMV23}. This already covers many familiar distribution classes. For our main result, however, we must handle \emph{arbitrary} learnable distribution classes, for which such a common dominating measure need not exist.  For instance, $\mathcal{P}=\{\delta_x: x\in \mathbb{R}\}$ is trivially learnable with a single sample but there is no $\sigma$-finite measure dominating all point masses. We, therefore, develop a more general version of the Poisson coupling that removes the common-dominating-measure assumption while retaining the same marginal and pairwise agreement guarantees on general standard Borel spaces.

The coupling methods allow us to turn closeness of TV distances between $q_i's$ to agreement probability lower bounds between $Z_i$'s. Private aggregation of $Z_i$'s then gives the private multi-sampler of Theorem~\ref{thm:sampling-informal}. In order to have a private distribution learner, one can simply feed a privatized sample set to the non-private learner. However, this approach would result in a poor dependency on $\beta$, the failure probability. To overcome this issue, we use a boosting argument. 

This results in the logarithmic dependence of sample complexity 
on $1/\beta$ in Theorem~\ref{thm:main-informal}.

\subsection{AI Disclosure}
The initial resolution for Ashtiani's conjecture in the case of countable domains was independently obtained by two authors prompting GPT-5.6 Sol. An internal Gemini model also produced a positive resolution of the conjecture by discovering a similar algorithm to the one presented. All authors digested the result and verified it to be correct.

Additional prompting, verification, and derivations from the authors resulted in the presented measure-theoretic arguments that consititute a full resolution to the conjecture for general domains. Care was taken to write the paper in a way that communicates the ideas most clearly to a human reader. The authors also conducted a literature search to attribute the main algorithmic ideas to~\cite{AngelS19} and~\cite{KalavasisKMV23}. 
All writing in this paper is done by the authors, who take full responsibility for the claims made herein.

\section{Preliminaries}

Let $\mathcal Y$ be a standard Borel space, and let $\mathcal B(\mathcal Y)$ denote the Borel sets and $\Delta(\mathcal Y)$ the probability measures on $\mathcal Y$,
with the $\sigma$-algebra generated by the map $P\mapsto P(B)$ for $B\in\mathcal B(\mathcal Y)$. Subsets and finite products carry the induced and product $\sigma$-algebras, respectively.

For a random variable $Z$, we write $\mathcal L(Z)$ for its law, i.e., $\mathcal L(Z)(B)=\Pr[Z\in B]$ for every measurable set $B$. For $P,Q\in\Delta(\mathcal Y)$, their total variation distance is $\mathrm{TV}(P,Q)
    \coloneqq \sup_{B\in\mathcal B(\mathcal Y)} |P(B)-Q(B)|$.

Our coupling constructions use two standard measure-theoretic objects.
First, we will use Markov kernels to formalize randomized mappings whose output distributions depend measurably on their inputs. In particular, they allow us to view our coupling procedures as randomized maps from tuples of distributions to joint synthetic datasets.

\begin{definition}[Markov kernel]\label{def:markov-kernel}
Let $(\mathcal X,\mathcal A)$ and
$(\mathcal Y,\mathcal B)$ be measurable spaces.
A \emph{Markov kernel}
$K:\mathcal X\rightsquigarrow\mathcal Y$
assigns to each $x\in\mathcal X$ a probability measure
$K(x;\cdot)$ on $\mathcal Y$, such that, for every
$B\in\mathcal B$, the map $x\mapsto K(x;B)$ is
$\mathcal A$-measurable.
\end{definition}
Second, our explicit coupling constructions are based on the Poisson coupling of Angel and Spinka~\cite{AngelS19}, which is built from Poisson point processes. We recall their basic definition below. 

Informally, a Poisson point process with intensity measure $\mu$ is a random cloud of points scattered over the space according to $\mu$, where the expected number of points in a region $B$ is $\mu(B)$, and the number of points falling in disjoint regions are independent.
\begin{definition}[Poisson point process]\label{def:PPP}
Let $(\mathcal Y,\mathcal B)$ be a measurable space and let $\mu$ be a
$\sigma$-finite measure on $\mathcal Y$. A \emph{Poisson point process}
$\mathcal A$ on $\mathcal Y$ with intensity measure $\mu$ is a random counting
measure such that:
\begin{enumerate}
    \item for every measurable $B\subseteq\mathcal Y$ with
    $\mu(B)<\infty$, $\mathcal A(B)\sim \operatorname{Poisson}(\mu(B))$
    \item for every collection of pairwise disjoint measurable sets
    $B_1,\ldots,B_m$, the random variables
    $\mathcal A(B_1),\ldots,\mathcal A(B_m)$ are independent.
\end{enumerate}
\end{definition}

We recall the celebrated definition of differential privacy, as well as its useful post-processing and parallel composition property. 
\begin{definition}[Differential privacy (DP)~\cite{DworkMNS06,DworkKMMN06}]
Let $(\mathcal X,\mathcal A)$ and
$(\mathcal Y,\mathcal B)$ be measurable spaces, and let $\varepsilon>0$ and
$\delta\in(0,1)$. A randomized algorithm
$A:\mathcal X^*\to\mathcal Y$ is
$(\varepsilon,\delta)$-differentially private if, for every
$n\in\mathbb N$, every pair of neighboring datasets
$S,S'\in\mathcal X^n$ differing in at most one element,
and every measurable set $B\subseteq\mathcal Y$, we have
\[
    \Pr[A(S)\in B]
    \le e^{\varepsilon}\Pr[A(S')\in B]+\delta,
\]
where the probabilities are over the internal randomness of $A$.
\end{definition}

\begin{proposition}[Post-processing]
\label{prop:post-processing}
Let $(\mathcal X,\mathcal A)$, $(\mathcal Y,\mathcal B)$, and
$(\mathcal Z,\mathcal C)$ be measurable spaces. Suppose
$A:\mathcal X^*\to\mathcal Y$ is $(\varepsilon,\delta)$-differentially private, and let $K:\mathcal Y\rightsquigarrow\mathcal Z$ be a data-independent Markov kernel. Then the randomized algorithm $K\circ A: S \mapsto Z$ with $Z \sim  K(A(S);\cdot)$ is also $(\varepsilon,\delta)$-DP.
\end{proposition}

\begin{lemma}[Parallel composition
\cite{McSherry09}]
\label{thm:parallel-composition}
Let $\varepsilon\ge0$, $\delta\in[0,1)$, and let
$I_1,\ldots,I_k$ be a fixed partition of $[n]$.
For each $i\in [k]$, suppose $A_i:\X^*\rightarrow \Y$ is
$(\varepsilon,\delta)$-differentially private. Then, for any $S\in \X^n$, the algorithm
$A(S)\coloneqq\bigl(A_1(S_{I_1}),\ldots,A_k(S_{I_k})\bigr)$,
where the invocations use independent randomness, is
$(\varepsilon,\delta)$-DP.
\end{lemma}

\section{Couplings}\label{sec:coupling}

As discussed in Section~\ref{sec:techniques}, our private aggregation procedure requires a correlated sampling mechanism that converts TV closeness of the learner outputs into exact agreement of the sampled synthetic datasets. We now formalize the properties of the coupling needed for the subsequent privacy and utility analyses. In addition to having the correct marginals and a small pairwise disagreement probability discussed above, we further require the joint law of every finite collection of coupled samples to depend measurably on the input distributions.
This ensures that the coupling can be applied to random data-dependent learner outputs and then supplied to subsequent randomized algorithms.\footnote{This notion is closely related to correlated sampling
\cite[Definition~7]{GhaziKM21}: a strategy with multiplicative
error $2$ and a common random seed induces a family satisfying
(G1) and (G2). We explicitly require (G3) to ensure measurable
dependence of the joint laws on the input distributions, enabling
composition with subsequent randomized algorithms on standard
Borel spaces.}
Throughout, we assume the domain $\Y$ to be standard Borel.

\begin{definition}[Grand coupling scheme]\label{def:grand-scheme}
Let $\mathcal Y$ be standard Borel and $\mathcal P \subseteq \Delta(\mathcal Y)$. A \emph{grand
coupling scheme} on $\mathcal P$ is a family of random variables $(Z_P)_{P \in \mathcal P}$, on a
common probability space, with joint law $\Gamma_{\mathcal{P}}$, such that:
\begin{enumerate}[label=\textup{(G\arabic*)}]
\item \emph{(Marginals.)} $Z_P \sim P$ for every $P \in \mathcal P$.
\item \emph{(Agreement.)} $\Pr[Z_P \ne Z_Q] \le 2\,\mathrm{TV}(P,Q)$ for all $P, Q \in \mathcal P$.

\item \emph{(Measurable marginal kernels.)} For every $k\ge 1$,
the map
$\mathbf P\mapsto\mathcal L(Z_{P_1},\dots,Z_{P_k})$
defines a Markov kernel
$\Gamma_k:\mathcal P^k\rightsquigarrow\mathcal Y^k$,
where $\mathbf P=(P_1,\dots,P_k)\in\mathcal P^k$.
\end{enumerate}
$\mathsf{Couple}_k(\mathbf P)$ is then \emph{defined} as a sample from $\Gamma_k(\mathbf P;\cdot)$.
\end{definition}

We present two constructions of this interface, based on the Poisson coupling of \cite{AngelS19}. Full details and complete proofs for both constructions are deferred to Appendix~\ref{app:coupling-details}.
The first, presented in Appendix~\ref{subsec:construction_1}, is the original construction taken verbatim from \cite{AngelS19}, which applies when the subset of relevant distributions $\mathcal{P}$ admits a common $\sigma$-finite dominating measure $\nu$.

We illustrate how Construction~\ref{cons:construction_1} works with a simple example.
Consider a set of distributions $\mathcal{P}$ on domain $[0,1]$, where for each distribution $P \in \mathcal{P}$, we have density $f_P$ with respect to the Lebesgue measure bounded by a constant $M$. 
Draw an infinite sequence of i.i.d.\ pairs $(Y_i, U_i)$, distributed uniformly over $[0,1]\times[0,M]$.
This sequence of pairs is shared across all $P \in \mathcal{P}$.
To draw a sample from a particular $P$, we output the $Y_i$ with the minimum $i$ such that $U_i \leq f_P(Y_i)$. 
Intuitively, to sample from a distribution $P$, this process outputs the first coordinate of the first point that lies under $P$'s density.
By a standard rejection sampling argument, the resulting output $Z_P$ has law $P$, giving the marginal property (G1). For $P,Q\in\P$, the areas of the intersection and union of their acceptance regions are $1-\TV(P,Q)$ and $1+\TV(P,Q)$, respectively. Hence, the first proposal accepted by either distribution is accepted by both with probability $\frac{1-\TV(P,Q)}{1+\TV(P,Q)}$, and thus yields the disagreement property (G2),  $\Pr[Z_P \ne Z_Q] \leq \frac{2\TV(P,Q)}{1+\TV(P,Q)}\leq 2 \TV(P,Q)$.

More precisely, Construction~\ref{cons:construction_1} yields a grand coupling scheme on $\mathcal{P}$, using a single Poisson point process shared across all $P\in\mathcal P$. 
Each $P \in \P$ defines its own acceptance region through its density $dP/d\nu$, and the coupled sample \(Z_P\) is the $y$-coordinate of the first point of the shared Poisson process that falls in that region, i.e., the point in the region with the smallest $t$-coordinate.

\begin{proposition}
[Informal version of Corollary~\ref{cor:cons1_grand_coupling_scheme}]\label{prop:informal_cons1_grand_coupling_scheme}
Suppose $\Y$ is standard Borel and there exists a $\sigma$-finite measure $\nu$ such that $P \ll \nu$ for all $P \in \mathcal{P}$. Then the family of random variables $(Z_P)_{P \in \mathcal{P}}$ defined in Construction~\ref{cons:construction_1} yields a grand coupling scheme $\Gamma_{\P}$ on $\mathcal{P}$.
\end{proposition}

Construction \ref{cons:construction_1} already captures a wide range of distribution classes. For example, a family of distributions on $\mathbb{R}^d$ that are all absolutely continuous with respect to Lebesgue measure admits a common $\sigma$-finite dominating measure. However, an arbitrary learnable distribution class need not admit such a measure. For instance, $\mathcal{P}=\{\delta_x: x\in \mathbb{R}\}$ is trivially learnable with a single sample but there is no $\sigma$-finite measure dominating all point masses. Therefore, the first construction can not handle such classes.

The key observation behind our second construction is that while a global measure $\nu$ may not exist over an arbitrary family, any finite tuple of distributions $\mathbf{P} = (P_1, \dots, P_k) \in \Delta(\Y)^k$ always admits a canonical finite dominating measure, $\nu_{\mathbf{P}} = \sum_{i=1}^k P_i$. 
 Returning to the simple example before, for any fixed set of $k$ distributions $\mathbf{P}$, we can imagine sampling candidate points $Y_i$ from a uniform mixture over $\mathbf{P}$.
We can therefore apply the same Poisson coupling construction to each finite tuple respectively, using the tuple-dependent dominating measure $\nu_{\mathbf P}$. This yields a coupling $\Gamma_k(\mathbf{P};\cdot) \Delta(\Y^k)$ of $P_1,...,P_k$. We then show that these couplings vary measurably with $\mathbf{P}$ with the standard Borel machinery, so that $\Gamma_k: \Delta(\mathcal Y)^k\rightsquigarrow\mathcal Y^k$ is indeed a Markov kernel. 

The remaining issue is that, the reference measure $\nu_{\mathbf{P}}$, and thus the Poisson process used in the construction, depend on the tuple $\mathbf{P}$ -- it essentially corresponds to a mixture over these $k$ distributions. 
In our later privacy and utility analyses, however, we will need our coupling to handle when we enlarge a tuple of $k$ distributions into a tuple of $k+1$ distributions without changing the joint distribution of the coordinates already present. This will capture adjoining the ground truth distribution in the utility analysis, or jointly realizing two neighboring executions in the privacy analysis.
A priori, the coupling constructed for $(P_1,P_2)$ need not coincide with the first two coordinates of the coupling constructed for $(P_1,P_2, P_3)$. Fortunately, the construction is insensitive to this choice. More formally, the finite-dimensional kernels are projectively consistent, namely, projecting the coupling of a larger tuple onto any subset of its coordinates gives exactly the coupling obtained by applying the construction directly to the corresponding subtuple.\footnote{Although the kernels $\Gamma_k$ are defined on ordered tuples, they are permutation equivariant (Proposition~\ref{prop:kernels}). Namely, the finite-dimensional law associated with a set of distributions is independent of the arbitrary ordering chosen for that set.}

Finally, Kolmogorov's Extension Theorem (Corollary~\ref{cor:ket_coupling}) then guarantees the existence of a single, universal joint measure $\Gamma_\infty\in\Delta(\Y^{\Delta(\Y)})$. In other words, $\Gamma_\infty$ jointly realizes a random variable $Z_P$ for every $P\in\Delta(\Y)$, with finite-dimensional marginals given by the kernels $\Gamma_k$, such that $Z_P\sim P$ for every $P\in\Delta(\Y)$, and every pair satisfies the desired disagreement bound. In particular, restricting this universal coupling to any $\mathcal P\subseteq\Delta(\Y)$ therefore yields a data-independent grand coupling scheme $\Gamma_{\P}$, without requiring a common dominating measure.

\begin{proposition}
[Informal version of Corollary~\ref{cor:cons2_grand_coupling} and~\ref{cor:ket_coupling}]\label{prop:informal_cons2_grand_coupling_scheme}
Suppose $\Y$ is standard Borel. The family of random variables $(Z_P)_{P \in\Delta(\mathcal{Y})} \sim \Gamma_\infty$ induced by Construction~\ref{cons:construction_2}x constitutes a universal grand coupling scheme on $\Delta(\Y)$. Consequently, on any subset $\P \subseteq \Delta(\Y)$, its restriction $(Z_P)_{P \in \mathcal{P}}$ yields a valid grand coupling scheme $\Gamma_{\P}$ on $\P$.
\end{proposition}

We emphasize that this grand coupling scheme serves purely as an analytical tool. Its purpose is to bring independent runs of the algorithms on two inputs into a shared probability space, for our privacy and utility analysis. By the nature of the coupling, the marginals exactly match the true target distributions (G1 in Definition~\ref{def:grand-scheme}). Consequently, all results derived within this joint probability space apply directly to the actual execution of the algorithms.

\section{Private Synthetic Data Sampler}\label{sec:private-data-sampler}
In this section, we apply the abstract grand coupling schemes developed in Section~\ref{sec:coupling} to construct a differentially private synthetic data sampler. Here, we set $\Y = \X^\ell$, where $\X$ is the data universe and $\ell$ is the size of the synthetic dataset.
Let $\mathcal F\subseteq\Delta(\mathcal X)$ contain
the ground-truth distribution class $\mathcal Q$ and all possible outputs of the non-private learner, and set
$\mathcal P=\{q^\ell:q\in\mathcal F\}$.
The goal is to privately output a synthetic dataset $Z \in \mathcal{X}^\ell$ of arbitrary target size $\ell$ such that $\mathrm{TV}(\mathcal{L}(Z), p^\ell) \le \beta$, given a dataset $S \sim p^n$ from an unknown distribution $p \in \Q$.

\paragraph{Technical Intuition.}

At a high level, we partition the samples into disjoint blocks, run the non-private learner on each block, and use our grand coupling to generate correlated synthetic datasets from the resulting learner outputs.

For the utility analysis, we would like to show not only that the synthetic datasets $Z_1,...,Z_k$ agree with one another frequently, but that their common values ``behaves like an i.i.d. sample'' from the ground truth distribution $p$. To this end, we introduce an analysis-only ``ghost'' sample  $Z^*\sim p^\ell$ from the ground truth distribution and consider the augmented vector $(Z_1,\dots,Z_k,Z^*)$ with joint law $\Gamma_{k+1}    \bigl((q_1^\ell,\dots,q_k^\ell,p^\ell);\cdot\bigr)$, which is
the corresponding finite-dimensional marginal of our grand
coupling $\Gamma_{\mathcal P}$ from
Proposition~\ref{prop:informal_cons2_grand_coupling_scheme}. Note that the augmented vector is drawn from the same grand coupling, so adjoining this extra ``ghost'' coordinate does not change the joint law of the synthetic datasets produced by the algorithm, i.e.,
its marginal on the first $k$ coordinates is exactly
$\Gamma_k((q_1^\ell,\dots,q_k^\ell);\cdot)$.

Moreover, by the disagreement bound, for every learn output $q_i$, $\Pr[Z_i\ne Z^*]
    \le 2\,\mathrm{TV}(q_i^\ell,p^\ell)
    \le 2\ell\,\mathrm{TV}(q_i,p)$.
Thus, when most blockwise learner outputs are accurate, many of the corresponding synthetic dataset agree with $Z^*$ exactly. With appropriate accuracy parameters,
we show that a large majority agree with $Z^*$, providing
the frequency gap required by $\mathsf{StabilityBasedMode}$
to output $Z^*$, with high probability.
The additional coordinate $Z^*$ is used only in the
analysis and is not generated by the algorithm.

For the privacy analysis, fix neighboring datasets $S,S'$ and use the same partition and base-learner randomness in both executions. Changing one record affects at most one block, so the resulting tuples $\q=(q_1,...,q_k)$ and $\q'=(q_1',...,q_k')$ differ in at most one coordinate $j$. For the analysis only, we take a single draw $(Z_P)_{P\in\P}\sim\Gamma_{\P}$ and define $\widetilde Z_i=Z_{q_i^\ell}$, and $\widetilde Z'_i=Z_{(q'_i)^\ell}$. By (G3) and the definition of $\mathsf{Couple}_k$, $(\widetilde Z_1,\dots,\widetilde Z_k)$ has joint law $\Gamma_k((q_1^\ell,\dots,q_k^\ell);\cdot)$, and $(\widetilde Z'_1,\dots,\widetilde Z'_k)$ has joint law $\Gamma_k(((q'_1)^\ell,\dots,(q'_k)^\ell);\cdot)$. These are exactly the joint laws used by the coupled sampling step in the two executions. For every unchanged coordinate $i\ne j$, we have $q_i=q'_i$, and hence $\widetilde Z_i=\widetilde Z'_i$ for every realization of the shared draw.\footnote{For example, suppose $\q=(q_1,q_2)$ and $\q'=(q_1',q_2')$ are obtained from the executions on $S$ and $S'$ respectively, with $q_1\ne q_1'$ and $q_2=q_2'$.
The unchanged coordinate $i=2$ is represented in both executions by the same coordinate $Z_{q_2^\ell}=Z_{q_2'^\ell}$ of the shared grand coupling $(Z_P)_{P\in\P}\sim\Gamma_{\P}$, and thus, $\tilde Z_2' = Z_{q_2'^\ell} = Z_{q_2^\ell} =\tilde Z_2$ point wise.
}

Thus, the two inputs to $\mathsf{StabilityBasedMode}$ are neighboring for every realization, so its privacy guarantee applies to each fixed coupling draw. 

\paragraph{Road Map.}
We present the formal construction in two stages to isolate the coupling randomness from the statistical randomness of the dataset partition and base learners.
\begin{enumerate}
    \item \textsc{Aggregate} (Algorithm~\ref{alg:coupled-sampler}): Operates on a fixed tuple $\mathbf{q} = (q_1, \dots, q_k) \in \mathcal{F}^k$. It draws $(Z_1, \dots, Z_k)$ via the grand coupling and aggregates them with the stability-based mode. Its privacy (under tuple-level change-one adjacency) follows directly from the construction of the grand coupling, while its utility follows from the disagreement bound and the accuracy of the stability-based mode mechanism.
    \item \textsc{Sampler} (Algorithm~\ref{alg:private-correlated-sampler}): Operates on raw datasets $S \in \mathcal{X}^n$, partitioning $S$ to form $\mathbf{q}$ before invoking \textsc{Aggregate}. It inherits the guarantees of Stage 1, and we can lift the privacy to standard change-one adjacency on $\mathcal{X}^n$ since changing one data point affects a single block in $\mathbf{q}$, while utility follows by combining Hoeffding's inequality over the block partition with the utility guarantees of \textsc{Aggregate}.
\end{enumerate}
We defer the detailed proofs and discussions in Appendix~\ref{app:sampler-details}.

\subsection{Tuple-Level Primitive: \textsc{Aggregate}}
We first present \textsc{Aggregate} (Algorithm~\ref{alg:coupled-sampler}), a mechanism operating on a fixed tuple of candidate distributions $\q = (q_1,\dots,q_k) \in \F^k$ to output a synthetic candidate $Z \in \X^\ell$. Analysis in this subsection is purely with respect to tuple-level adjacency, where neighboring inputs $\mathbf{q}, \mathbf{q}'$ differ in a single distribution coordinate $q_j$.

\begin{algorithm}
\caption{\textsc{Aggregate}}
\label{alg:coupled-sampler}
\begin{algorithmic}[1]
\Require Distributions $q_1,\ldots,q_k \in \F$.
\Ensure Synthetic dataset $Z \in \X^\ell$
\State Jointly sample $(Z_1,...,Z_k) \gets \mathsf{Couple}_k(q_1^\ell,\ldots,q_k^\ell)$
{\color{blue}{\noindent\texttt{// $\mathsf{Couple}_k$ is a finite-dimensional marginal of a valid grand coupling scheme on $\mathcal{P} = \{q^\ell : q \in \mathcal{F}\}$, which can be instantiated by either $\mathsf{Couple\_I}$ in Algorithm~\ref{alg:coupled-I} or $\mathsf{Couple\_II}$ in Algorithm~\ref{alg:coupled-II}}}}
\State $Z= \mathsf{StabilityBasedMode}(Z_1,...,Z_k)$
\State \Return $Z$.
\end{algorithmic}
\end{algorithm}

By instantiating $\mathsf{Couple}_k$ with a valid grand coupling scheme $\Gamma_{\P}$\footnote{e.g., Construction~\ref{cons:construction_1} if a common dominating measure exists, or Construction~\ref{cons:construction_2} for general standard Borel spaces}, it produces a joint measurable vector $(Z_1,\dots,Z_k)$, for any input tuple $\q$, such that $Z_i \sim q_i^\ell$ and $\Pr[Z_i \ne Z_j] \le 2\TV(q_i^\ell, q_j^\ell)$ for all $i \ne j$.
We next show that: (i) the privacy guarantee of $\mathsf{StabilityBasedMode}$ transfers directly to \textsc{Aggregate}, and (ii) if a large majority of the input distributions $q_i$ are close to a target ground-truth distribution $p$ in TV, the output of \textsc{Aggregate} matches an ideal draw from $ p^\ell$ with high probability.

\begin{theorem}[Informal version of Corollary~\ref{cor:instantiated-coupled-sampler-utility}]
\label{thm:informal-aggregate}
For every $\varepsilon>0$, $\delta,\beta\in(0,1)$,
$\ell\in\mathbb N^+$, and
$k=\Omega\left(\frac{\log(4/\beta)+\log(1/\delta)}{\varepsilon}\right)$,
\textsc{Aggregate} is $(\varepsilon,\delta)$-DP~\footnote{With respect to tuple-level adjacency, where
neighboring inputs differ in a single distribution $q_i$.}.
Moreover, for any $p\in\Q$ and input tuple $\q\in\F^k$,
if at least $7/8$ of the input distributions are
$O(\beta/\ell)$-close to $p$ in TV, then the law of its
output is $\beta/2$-close to $p^\ell$ in TV.
\end{theorem}

\begin{proof}[Proof Sketch]
\textbf{Privacy:} Since the joint law of $\mathsf{Couple}_k(\q)$ is a finite-dimensional marginal of the global measure $\Gamma_{\P}$\footnote{For $\mathsf{Couple\_I}$ (Construction~\ref{cons:construction_1}), the joint law of $(Z_1,\dots,Z_k)\sim \mathcal{L}(\mathsf{Couple}_k(\q))$ is the $(q_1^\ell,\dots,q_k^\ell)$-coordinate marginal of the explicit coupling $\Gamma$ on $\P=\{q^\ell : q \in \F\}$. For $\mathsf{Couple\_II}$ (Construction~\ref{cons:construction_2}), it is the corresponding finite-dimensional marginal of the grand coupling $\Gamma_\infty$ on $\mathcal{Y}^{\Delta(\mathcal{Y})}$, where $\mathcal{Y}=\mathcal{X}^\ell$.}, we evaluate executions on neighboring tuples $\q$ and $\q'$ on a single shared probability space. Specifically, we construct an analysis coupling from a single draw $(Z_P)_{P \in \P} \sim \Gamma_{\P}$, setting $\tilde{Z}_i = Z_{q_i^\ell}$ for input $\q$ and $\tilde{Z}'_i = Z_{(q_i')^\ell}$ for input $\q'$. For every unchanged coordinate $i \neq j$, $q_i = q_i'$ implies $Z_{q_i^\ell} = Z_{(q_i')^\ell}$ pointwise. Consequently, the realized input vectors $(\tilde{Z}_1, \dots, \tilde{Z}_k)$ and $(\tilde{Z}'_1, \dots, \tilde{Z}'_k)$ differ in at most a single coordinate $j$ for \emph{every realization} of $\Gamma_{\P}$. Therefore, since neighboring input tuples yield pointwise neighboring vectors under $\Gamma_{\P}$, taking expectations over $\Gamma_{\P}$ transfer the DP guarantee of $\mathsf{StabilityBasedMode}$ directly to \textsc{Aggregate}.

\textbf{Utility:} 
As in the privacy analysis, the utility guarantee relies on constructing an analysis coupling from the grand coupling scheme $\Gamma_{\P}$. Specifically, for input distributions $\q = (q_1, \dots, q_k) \in \F^k$ and a target distribution $p \in \F$, we take a single draw $(Z_P)_{P \in \P} \sim \Gamma_{\P}$ and define $\tilde{Z}_i = Z_{q_i^\ell}$ for each $i \in [k]$, alongside an unobserved ``ghost'' sample $\tilde{Z}_* = Z_{p^\ell}$, indexed by the ground-truth product distribution $p^\ell \in \P$. Since $\mathcal{L}(\mathsf{Couple}_k(\mathbf{q}))$ and $\mathcal{L}(\mathsf{Couple}_1(\mathbf{p}))$ are simply a finite-dimensional marginal of the global measure $\Gamma_{\mathcal{P}}$, our construction from the grand coupling scheme $\Gamma_{\P}$ satisfies $(\tilde{Z}_1,...,\tilde{Z}_k) \sim \mathcal{L}(\mathsf{Couple}_k(\mathbf{q}))$ and $\tilde{Z}_* \sim \mathcal{L}(\mathsf{Couple}_1(p))=p^\ell$, by property G3 and G1 from Definition~\ref{def:grand-scheme}).  

For every accurate input distribution where $\TV(q_i, p) \le \eta$, property G2 bounds the disagreement by $\Pr[\tilde{Z}_i \neq \tilde{Z}_*] \le 2\TV(q_i^\ell, p^\ell) \le 2\ell\eta$. When $7/8$ of the $q_i$ are $\eta$-close to $p$, an overwhelming majority of the coupled samples $\tilde{Z}_i$ collapse onto the exact same realization $\tilde{Z}_* \sim p^\ell$ pointwise. This manufactures the frequency gap required for $\mathsf{StabilityBasedMode}$ \cite{ts13, Vadhan17} to output $\tilde{Z}_*$ with probability $\ge 1 - \delta_0$. 

We emphasize that neither $\tilde{Z}_*$ nor $\tilde{Z}'_i$ is computed by the algorithm, and they are introduced purely for the analysis. Full technical details appear in Proposition~\ref{prop:coupled-sampler-utility}, Proposition~\ref{prop:sampler-privacy}, and Corollary~\ref{cor:instantiated-coupled-sampler-utility}.
\end{proof}

\begin{remark}[Comparison to \cite{KalavasisKMV23}]
    Our framework shares structural roots with the reduction of \cite{KalavasisKMV23} from TV-indistinguishable to differentially private PAC learning, which likewise runs a learner on disjoint blocks and jointly samples the $k$ outputs via the Poisson coupling of \cite{AngelS19}. 
    The two constructions diverge at the aggregation stage due to the different setups. 

    The TV-indistinguishability guarantee in \cite{KalavasisKMV23} provides only a constant level of agreement between coupled learner outputs. This is sufficient to ensure the existence of frequently occurring candidates, but not necessarily a single overwhelmingly dominant one. They therefore first isolate candidate clusters using a stable histogram and subsequently score each candidate's empirical risk using another exponential mechanism.
    
    In our setting, we can exploit the stronger assumption that the distribution
    class is learnable to \emph{arbitrarily small} TV error. Specifically, we can tune the block sample size $b$ in a way that most of the learned distribution $q_i$ are close to the underlying ground truth distribution $p$ with TV error $\eta$ for $\eta$ as small as needed. Thus, for sufficiently samll $\eta$ such that $\TV(q_i^\ell, p^\ell) \le 2\ell\eta \ll 1$, adjoining an unobserved ground-truth draw $Z^* \sim p^\ell$ under the grand coupling forces an overwhelming majority of candidate draws to collapse onto the exact same realization $Z^*$. This manufactured consensus allows a standard $\textsf{StabilityBasedMode}$ to output $Z^*$ directly. 
    Thus, our stronger learnability assumption allows us to replace the stablehist-then-privselect aggregation used in \cite{KalavasisKMV23} by a single stability-based mode.

    Furthermore, \cite{KalavasisKMV23} already propose in Appendix~D.5 using a tuple-dependent mixture of the learner-output distributions as the reference measure to move beyond countable domains, while leaving the corresponding privacy guarantee open. Our Construction~\ref{cons:construction_2} builds on this idea and resolves the missing technical issues, by establishing the measurability and consistency properties of the resulting tuple-dependent Poisson couplings, which yields a universal grand coupling on arbitrary standard Borel spaces and, in turn, a valid privacy analysis without assuming a common dominating measure.
\end{remark}

\subsection{End-to-End Private Correlated Sampler: \textsc{Sampler}}
In the preceding subsection, we established the privacy and utility guarantees of \textsc{Aggregate} conditioned on a fixed tuple of candidate distributions $\q = (q_1, \dots, q_k)$. We now lift these guarantees to the full end-to-end pipeline, \textsc{Sampler} (Algorithm~\ref{alg:private-correlated-sampler}), where input distributions are learned from disjoint blocks of a raw dataset $S = (X_1, \dots, X_n) \in \X^n$. The algorithm partitions $S$ into disjoint blocks, runs the base learner on each block to obtain the random tuple $\q$, and then applies \textsc{Aggregate}. Here, privacy is naturally defined at the dataset level, where two datasets $S, S' \in \mathcal{X}^n$ are adjacent if they differ in a single record.

\begin{algorithm}[t]
\caption{\textsc{Sampler}}
\label{alg:private-correlated-sampler}
\begin{algorithmic}[1]
\Require Dataset $S=(X_1,\ldots,X_n)$ with $n=bk$. Number of blocks $k$. 
\Ensure Synthetic samples $Z\in\X^\ell$
\State Let a random seed (a data-independent permutation) be $\Pi \sim \mathrm{U}(\mathcal{S}_n)$
\State Partition $S$ into $k$ disjoint blocks $S_1,\ldots,S_k\in\X^b$, by setting $S_i \gets (X_{\Pi((i-1)b+1)},...,X_{\Pi(ib)})$
\State Let learners' random seeds (data- and permutation-independent) be $U_1,...,U_k \overset{\text{i.i.d.}}{\sim} \mathrm{U}([0,1])$.
\State For each $i\in[k]$, compute $q_i\gets L(S_i;U_i)$
\State $Z\gets\textsc{Aggregate}(q_1,\ldots,q_k)$
\State \Return $Z\in\X^\ell$
\end{algorithmic}
\end{algorithm}

While much of the analysis relies on standard techniques, the primary technical step lies in explicitly controlling the sources of randomness to bridge $\textsc{Aggregate}$ and $\textsc{Sampler}$, ensuring that the resulting random tuple satisfies the conditions required by $\textsc{Aggregate}$.

\begin{theorem}[Informal version of Corollary~\ref{cor:instantiated-private-correlated-sampler-utility}]
\label{thm:informal-sampler}
Fix a base learner $L$ for $\Q$.
For every $\varepsilon,\delta,\beta\in(0,1)$ and
$\ell\in\mathbb N^+$, let $b$ be a sample size such
that, for every $p\in\Q$, $L$ achieves
$\TV(L(S_i),p)\le O(\beta/\ell)$ with probability at
least $15/16$ given $S_i\sim p^b$.
Then,
\textsc{Sampler} satisfies dataset-level
$(\varepsilon,\delta)$-DP.
Moreover, for every $p\in\Q$, given $S\sim p^n$ with $n=\Omega\left(\frac{b\log(4/\delta\beta)}{\varepsilon}\right)$,
it outputs a synthetic dataset $Z\in\mathcal X^\ell$
satisfying $\TV(\mathcal L(Z),p^\ell)\le 3\beta/4$.
Thus, \textsc{Sampler} is an $(\ell,3\beta/4)$-approximate sampler
for $\Q$ in the sense of Definition~\ref{def:sampler}.
\end{theorem}

\begin{proof}[Proof Sketch]

\textbf{Privacy:} Fix adjacent datasets $S, S' \in \X^n$ differing at record $j$. For any realization of the permutation $\Pi$, the  record that differs in the two worlds falls into a unique block index $a$. Since all other blocks, $S_i$ for $i \neq a$, remain identical, and the base learner's internal randomness $U_1, \dots, U_k$ is drawn independently of input data $S$, the output tuples $\q = L(S)$ and $\q' = L(S')$ differ in at most coordinate $a$. Taking an expectation over the data-independent randomness, i.e., the random permutation for data partition and the learners' random seeds, the tuple-level DP guarantee of \textsc{Aggregate}, shown in Theorem~\ref{thm:informal-aggregate}, directly lifts to dataset-level $(\varepsilon, \delta)$-DP for \textsc{Sampler}.

\textbf{Utility:} Define indicator variables $I_i = 1\{\TV(q_i, p) \le \eta\}$ for $\eta = O(\beta/\ell)$. Since $S \sim p^n$ i.i.d., the blocks $S_i$ are i.i.d.\ draws from $p^b$, making $I_1, \dots, I_k$ i.i.d.\ Bernoulli variables with mean $\E[I_i] \ge 15/16$. By Hoeffding's inequality, setting $k = O(\log(1/\beta))$ ensures that $\sum_{i=1}^k I_i \ge \frac{7}{8}k$ holds with probability at least $1 - \beta/4$. Conditioning on this $\text{good event}=|\{\TV(q_i, p) \le \eta\}| \geq \frac{7}{8}k$, Theorem~\ref{thm:informal-aggregate} guarantees that the output distribution $\mathcal{L}(Z \mid \text{good event})$ satisfies $\TV(\mathcal{L}(Z \mid \text{good event}), p^\ell) \le \beta/2$. Combining the failure probability of the concentration with the conditional TV error yields $\TV(\mathcal{L}(Z), p^\ell) \le \beta/2 + \beta/4 = \frac{3}{4}\beta$.

Full technical details appear in Proposition~\ref{prop:PCS-privacy} and Proposition~\ref{prop:PCS-utility}.
\end{proof}

\begin{remark}
\label{rem:private_pac_learning}
Since private and non-private PAC-learnability are known to not be equivalent~\cite{BunNSV15, alon2019private}, it is instructive to understand where the present approach fails in supervised settings. This becomes clear by looking at Algorithm~\ref{alg:private-correlated-sampler}. The algorithm crucially relies on the fact that a distribution learning algorithm outputs a density that is close to the ground truth in TV-distance. However, in supervised settings, the input consists of domain elements and their labels, while the learner outputs a function attempting to capture the rule that associates the former to the latter. This does not capture the distribution of the unlabeled examples, and therefore cannot be used as the basis for a sampler.
\end{remark}

\section{Main algorithm and reduction}

We now present our main application, a black-box reduction that transforms any non-private base learner $L$ into a DP learner via \textsc{Sampler}. Let $m(\alpha, \beta)$ denote the non-private sample complexity of $L$ for outputting an $\alpha$-accurate estimate in terms of TV, with probability at least $1 - \beta$.

The algorithm is conceptually simple, where we first invoke \textsc{Sampler} to generate a private synthetic dataset $Z \in \mathcal{X}^\ell$, and then apply the non-private learner $L$ to the private synthetic dataset $Z$ using an independent random seed $V$.

The complete procedure is described in Algorithm~\ref{alg:private-learner}. 

The results follows immediately from the guarantees of \textsc{Sampler}. Privacy follows by post-processing, and utility follows by coupling its output to an i.i.d.\ sample from $p^\ell$ and union-bounding the coupling mismatch with the learner's failure event.
The main privacy and utility guarantees of \textsc{PrivateLearner} are established in the following theorem.

\begin{algorithm}[t]
\caption{\textsc{BlackBoxPrivateLearner}}
\label{alg:private-learner}
\begin{algorithmic}[1]
\Require Dataset $S=(X_1,\ldots,X_n)$
\Ensure A distribution $\hat{q} \in \Delta(X)$
\State $Z\gets\textsc{Sampler}(S)$ 
\State Pick a learner's random seed (data-independent) $V \sim \mathrm{U}([0,1])$ and $\hat{q} \gets L(Z;V)$
\State \Return $\hat{q}$
\end{algorithmic}
\end{algorithm}

\begin{theorem}[Main Reduction]
\label{thm:main}
Let $\alpha,\beta,\delta,\eps\in(0,1)$ be parameters, and $L$ be a non-private learner for $\Q$ with sample complexity $m(\alpha,\beta)$.
Then, \textsc{BlackBoxPrivateLearner}
(Algorithm~\ref{alg:private-learner}) is
$(\varepsilon,\delta)$-DP.
Moreover, there is a sample
size
$$n=O\left(
m\left(\frac{\beta}{64m\left(\alpha,\frac{\beta}{4}\right)},
\frac{1}{16}\right)\frac{1}{\varepsilon}
\log\left(\frac{4}{\beta\delta}\right)
\right)$$
such that, for every $p\in\Q$, given $S\sim p^n$,
the output $\hat q$ satisfies $\TV(\hat q,p)<\alpha$
with probability at least $1-\beta$.
\end{theorem}
\begin{proof}
Set $\ell=m\left(\alpha,\frac{\beta}{4}\right)$,
$\eta=\frac{\beta}{64\ell}$,
$b=m\left(\eta,\frac{1}{16}\right)$, and
$\delta_0=\min\left\{\delta,\frac{\beta}{4}\right\}$.
By Proposition~\ref{prop:PCS-privacy},
\textsc{Sampler} is $(\varepsilon,\delta_0)$-DP.
Applying $L$ to its output using independent randomness
preserves privacy by post-processing (Proposition~\ref{prop:post-processing}).
Since $\delta_0\le\delta$, the resulting learner is
$(\varepsilon,\delta)$-DP.

For utility, take
$k=\left\lceil\max\left\{
\frac{8}{\varepsilon}\log\left(\frac{4}{\beta\delta}\right),
128\log\left(\frac{4}{\beta}\right)
\right\}\right\rceil$
and $n=bk$.
This choice satisfies both the stability requirement
and the concentration bound in
Corollary~\ref{cor:instantiated-private-correlated-sampler-utility}.
Hence, for every $p\in\Q$ and $S\sim p^n$, the synthetic
output $Z$ satisfies
$\TV\left(\mathcal L(Z),p^\ell\right)\le\frac{3\beta}{4}$.
By the choice of $\ell$, the learner fails to achieve
TV error strictly below $\alpha$ on an i.i.d.\ sample
from $p^\ell$ with probability at most $\frac{\beta}{4}$.
Transferring this failure bound using union bound across the TV distance
gives
$\Pr\left[\TV(\hat q,p)\ge\alpha\right]
\le\frac{\beta}{4}+\frac{3\beta}{4}=\beta$.

Finally, substituting the value of $b$ and $k=O\left(\frac{1}{\varepsilon}
\log\left(\frac{4}{\beta\delta}\right)\right)$ into $n=bk$ gives the
claimed sample complexity.
\end{proof}

\section{A Boosted Private Learner}\label{sec:boosted-learning}

In the previous section, the \textsc{PrivateLearner} requires the block-level target TV error to satisfy $\eta \lesssim \beta / \ell$. When the non-private sample complexity scales polynomially with accuracy, i.e., $m(\alpha, \beta) = \mathrm{poly}(1/\alpha, \log(1/\beta))$, setting $\eta \propto \beta$ introduces an undesirable polynomial dependence $1/\beta^{\Omega(1)}$ into the overall sample complexity $n$. To get rid of this polynomial dependence and achieve a better logarithmic overhead $O(\log(1/\beta))$, we introduce a boosted version, \textsc{BoostedLearner} in Algorithm~\ref{alg:robust-learner}. 

The core idea is confidence boosting: amplifying a learner's
constant success probability to $1-\beta$ through independent
repetition and robust aggregation. We run \textsc{PrivateLearner} independently on
$T=O(\log(1/\beta))$ disjoint blocks $S_1,\dots,S_T$
with a fixed constant success probability $1-\beta_0$
(e.g., $\beta_0=1/8$), obtaining estimates
$\hat p_1,\dots,\hat p_T$.
This requires a block accuracy of only $\eta\sim1/\ell$,
independent of the final target failure probability $\beta$. For each $t\in[T]$, $N_t = |\{i \in[T]: \TV(\hat p_i, \hat p_t) \leq 2\alpha_0\}|$ counts how many estimates are within TV distance $2\alpha_0$ of $\hat p_t$, where $\alpha_0=\alpha/3$.
We select $j$ maximizing $N_j$ and return $\hat p_j$.

By doing so, if more than half of the estimates are $\alpha_0$-close
to $p$, they are all within $2\alpha_0$ of one another.
Thus, every such estimate has $N_t>T/2$, and the maximizing estimate also has $N_j>T/2$.
Among the estimates counted by $N_j$, at least one must therefore be $\alpha_0$-close to $p$, and thus
$\TV(\hat p_j,p)\le 2\alpha_0+\alpha_0=\alpha$.
Each run fails with probability at most $\beta_0$, and by independence, the probability of failing to obtain a majority of accurate estimates decreases exponentially
in $T$. See \ref{sec:proof:boosted-learner} for the formal proof.

\renewcommand{\algorithmicrequire}{\textbf{Input:}}
\renewcommand{\algorithmicensure}{\textbf{Output:}}
\begin{algorithm}[t]
\caption{\textsc{BoostedLearner}}\label{alg:robust-learner}
\begin{algorithmic}[1]
\Require Dataset $S=(X_1,\ldots,X_n)$, a non-private learner $L$, target accuracy $\alpha$, target failure probability $\beta$, and privacy parameters $(\varepsilon, \delta)$.
\Ensure A distribution $\hat q \in \Delta(\X)$
\State Set target base parameter $\alpha_0 \gets \alpha / 3$, constant failure rate $\beta_0 \gets 1/8$, privacy target $\delta_0 \gets \min\{\delta, \beta_0/4\}$.
\State Set partition count $T \gets \left\lceil 128 \ln(1/\beta) \right\rceil$
\State Partition $S$ into $T$ disjoint subsets $S_1, \dots, S_T \in \mathcal{X}^c$, each of size $c = O\left( \frac{\log(1/(\beta_0 \delta_0))}{\varepsilon} \cdot m\left( \frac{\beta_0}{64 \, m(\alpha_0, \beta_0/4)}, \frac{1}{16} \right) \right)$

\For{each $t \in [T]$ in parallel}
    \State Obtain base estimate: $\hat{p}_t \gets \textsc{PrivateLearner}(S_t, \alpha_0, \beta_0, \varepsilon, \delta_0)$.
\EndFor
\State For each $t \in [T]$, compute candidate ball size, $N_t \gets \left\vert{} \{ i \in [T] : \mathrm{TV}(\hat{p}_t, \hat{p}_i) \le \frac{2}{3}\alpha \} \right\vert{}$.
\State Find $j = \arg\max_{i \in [T]} N_i$
\State \Return $\hat{p}_j$
\end{algorithmic}
\end{algorithm}

\begin{theorem}[Boosted Private Learner] \label{thm:robust-learner} Let $\alpha,\beta,\varepsilon,\delta\in(0,1)$ be parameters, and $L$ be a non-private learner for $\Q$ with sample complexity $m(\alpha,\beta)$. Then, \textsc{BoostedLearner} (Algorithm~\ref{alg:robust-learner}) is $(\varepsilon,\delta)$-DP. Moreover, there is a sample size 
$$n=O\left( m\left( \frac{1}{512m\left(\alpha/3,1/32\right)}, \frac{1}{16} \right) \frac{1}{\varepsilon} \log\left(\frac{4}{\delta}\right) \log\left(\frac{4}{\beta}\right) \right)$$ 
such that, for every $p\in\Q$, given $S\sim p^n$, the output $\hat q$ satisfies $\TV(\hat q,p)<\alpha$ with probability at least $1-\beta$. In particular, if $m(\alpha,\beta)=\operatorname{poly}\left( \frac{1}{\alpha},\log\left(\frac{1}{\beta}\right)\right)$, then $n=\operatorname{poly}\left( \frac{1}{\alpha},\log\left(\frac{1}{\beta}\right), \frac{1}{\varepsilon},\log\left(\frac{1}{\delta}\right) \right)$.
\end{theorem}

Note that the DP guarantee is non-trivial only when $\delta = o(1/n)$, since allowing $\delta = \Omega(1/n)$ permits the learner to randomly select a user's data to output each run~\cite{McSherry17}. Our sample complexity bound is meaningful in this regime because of its dependence on $\delta$ only through a $\log(1/\delta)$ factor. Specifically, rewriting the sample complexity in Theorem~\ref{thm:robust-learner} as $n(\delta) = C \log(4/\delta)$, where $C$ collects other factors independent of $\delta$, we have $\delta \cdot n(\delta) = C \delta \log(4/\delta) \to 0$ as $\delta \to 0$, that is, $\delta = o(1/n)$. In contrast, a dependence of the form $n(\delta) = C/\delta$ would give $\delta \cdot n(\delta) = C$, so $\delta = \Theta(1/n)$, which falls back to the trivial regime.

\section*{Acknowledgements} We thank David Woodruff for discussions and assistance with the internal version of Gemini used. We also thank Matthew Joseph and Chris Liaw for their comments on a preliminary draft of this paper.

\bibliographystyle{alpha}
\bibliography{biblio}

\appendix
\section{Poisson Coupling Details}\label{app:coupling-details}

We present two constructions of the grand coupling interface in Definition~\ref{def:grand-scheme}, both based on the Poisson coupling of \cite{AngelS19}.
Subsection~\ref{subsec:construction_1} treats families $\mathcal{P}\subseteq\Delta(\mathcal{Y})$ that admit a common $\sigma$-finite dominating measure. For such families, Corollary~\ref{cor:cons1_grand_coupling_scheme} establishes a grand coupling scheme on $\mathcal{P}$.

Subsection~\ref{subsec:construction_2} treats the general case, without assuming a common dominating measure. For every standard Borel space $\mathcal{Y}$, Corollary~\ref{cor:cons2_grand_coupling} establishes a universal grand coupling scheme on $\Delta(\mathcal{Y})$. Restricting this scheme to any family $\mathcal{P}\subseteq\Delta(\mathcal{Y})$ gives a grand coupling scheme on $\mathcal{P}$.

\subsection{A Globally Dominated Poisson Coupling}\label{subsec:construction_1}
We first formalize the setting where all distributions in the family are dominated by a single reference measure. Let $(\Y,\mathcal{B})$ be a standard Borel space and let $\mathcal{P}\subseteq \Delta(\mathcal{Y})$ be a collection of probability measures of interest. We use the original construction taken verbatim from \cite{AngelS19}.

\begin{construction} \label{cons:construction_1}
    Suppose there exists a $\sigma$-finite measure $\nu$ such that $P\ll \nu$ for every $P \in \mathcal{P}$. We write $f_P = \frac{dP}{d\nu}$ for all $P \in \mathcal{P}.$

    Let $\A$ be a Poisson point process on $\Y \times \R_+ \times \R_+$ with intensity $\nu(dy)ds dt$ (Definition~\ref{def:PPP}). For a realization $\mathrm{a}$ of the Poisson point process $\A$, we let $\mathrm{a}_P =\{(y,s,t) \in \mathrm{a}: s \le f_P(y)\}$ for all $P \in \mathcal{P}$, and let $\Phi(P,\mathrm{a})$ denote the $y$-coordinate of the point in $\mathrm{a}_P$ having the smallest $t$-coordinate.  More formally, if  $(y^*,s^*,t^*)$ is the unique point satisfying $t^* = \min\{t: (y,s,t) \in \mathrm{a}, s \le f_P(y)\}$, then $\Phi(P,\mathrm{a})=y^*$. Note that on the null event that there is a tie we can define $\Phi(P,\mathrm{a})$ arbitrarily.\footnote{Under the law of PPP, such events have probability zero, so it does not affect any of the distributions. See proof of Prop 5 in \cite{AngelS19}.}. 
    Using the same random process $\A$ for every $P \in \mathcal{P}$, we define the random variable $Z_P = \Phi(P,\A)$ and let $\Gamma=\mathcal L((Z_P)_{P\in\mathcal P})$
be their joint law..
\end{construction}
\begin{algorithm}
\caption{$\mathsf{Couple\_I}$ from construction~\ref{cons:construction_1}}
\label{alg:coupled-I}
\begin{algorithmic}[1]
\Require Distributions $q_1,\ldots,q_k \in \F$. \textit{A fixed $\sigma$-finite measure $\nu$ on $\X^\ell$ dominating every relevant $q^\ell \in \{q^\ell: q \in\F\}$}.
\Ensure Synthetic samples $(Z_1,...,Z_k)$ such that $Z_i \sim q_i^\ell$ and $\Pr[Z_i \ne Z_j] \le 2 \TV(q_i^\ell, q_j^\ell)$ for all $i,j\in[k]$.
\State Let a random seed be $\A$, a Poisson point process with intensity $\nu(dy) ds dt$
\State $Z_i = \Phi(q_i^\ell, \A)$ for all $i\in[k]$ according to construction~\ref{cons:construction_1}
\Comment{\textit{Conditional on a realization $\A=\mathrm{a}$, i.e. conditional on a realization of the PPP, the output $z_i = \Phi(q_i^\ell, \mathrm{a})$ is deterministic.}}

\State \Return $(Z_1,...,Z_k)$.
\Comment{\textit{Equivalently, $(Z_1,...,Z_k)$ has the $(q_1^\ell,...,q_k^\ell)$-coordinate marginal of the coupling $\Gamma$ from construction~\ref{cons:construction_1}}}

\end{algorithmic}
\end{algorithm}

\begin{theorem}[Proposition 5 and Theorem 2 in \cite{AngelS19}]\label{thm:AS19_prop5}
Suppose $\Y$ is standard Borel and there exists a $\sigma$-finite measure $\nu$ such that $P \ll \nu$ for all $P \in \mathcal{P}$. The law $\Gamma$ defined in Construction~\ref{cons:construction_1}, is a coupling of $(P)_{P \in \mathcal{P}} $, i.e. $Z_P \sim P$ for every $P\in \mathcal{P}$. Moreover, for every $P,Q \in \mathcal{P}$, 
\begin{align*}
    \Pr[Z_P \ne Z_Q] \le \frac{2\TV(P,Q)}{1+\TV(P,Q)} \le 2 \TV(P,Q)
\end{align*}
\end{theorem}

Theorem~\ref{thm:AS19_prop5} establishes the core marginal and agreement guarantees for the coupling construction. However, to formally show that it constitutes a grand coupling scheme on $\P$, we must also verify the measurability requirement, as formalized below.

\begin{corollary}[Grand coupling scheme $\Gamma_{\mathcal{P}}$ from Construction~\ref{cons:construction_1}]\label{cor:cons1_grand_coupling_scheme}
Suppose $\Y$ is standard Borel and there exists a $\sigma$-finite measure $\nu$ such that $P \ll \nu$ for all $P \in \mathcal{P}$. Then the family of random variables $(Z_P)_{P \in \mathcal{P}}$ defined in Construction~\ref{cons:construction_1} constitutes a grand coupling scheme $\Gamma_{\P}$ on $\mathcal{P}$.
\end{corollary}
\begin{proof}
G1 and G2 follow immediately from Theorem~\ref{thm:AS19_prop5}. It suffices to show G3 that the map $\mathbf{P} \mapsto \mathcal{L}(Z_{P_1},...,Z_{P_k})$, where $\mathbf{P} = (P_1,...,P_k)$ with $P_i \in \P$, is a Markov kernel. 

Since $\Y$ is standard Borel and $P \ll \nu$ for all $P\in \P$, the Radon-Nikodym derivatives may be chosen jointly measurably in $(P,y)$ \cite{Kallenberg02}, so write $f(P_i,y)=\frac{dP_i}{d\nu}(y)$.  Let $\A$ be a Poisson point process on $\Y \times \R_+ \times \R_+$ with intensity $\nu(dy) ds dt.$ For every $i\in[k]$, define $\A_{i} = \{(y,s,t) \in \A: s \le f(P_i,y) \}$.  We define the random variable $Z_i$ by $Z_i = y$, where $(y,s,t)$ is the point of $\A_{i}$ with the minimal $t$ among all points of $\A_{i}$. Equivalently, we write $Z_i=\Phi(P_i,\A),$ where $\Phi$ denotes this deterministic rule. 

Now define $\Phi(\mathbf{P},\A)=(\Phi(P_1,\A), \ldots, \Phi(P_k,\A))$ and for measurable $B \subseteq \Y^k$, let $\Gamma_k(\mathbf{P};B)=\Pr_{\A}[\Phi(\mathbf{P},\A) \in B]$, i.e. $\Gamma_k(\mathbf P;\cdot) = \mathcal{L}(\Phi(\mathbf{P},\A))$. Since $f(P_i, y)$ is jointly measurable, the indicator $1\{s \le f(P_i,y) \}$ is also jointly measurable; applying a measurable filter to a point process yields another valid point process, and applying the argmin on $t$-coordinate is also measurable; and thus, the deterministic map $(\mathbf{P},a)\mapsto \Phi(\mathbf{P},a)$ is measurable. 
Therefore, $\mathbf{P} \mapsto \Gamma_k(\mathbf{P};B)$ is also measurable for every measurable $B\subseteq \Y^k$, so $\Gamma_k$ is a Markov kernel.

Thus, the joint law of $(Z_P)_{P \in \mathcal{P}}$ yields the desired grand coupling $\Gamma_{\P}$ on $\mathcal{Y}^{\mathcal{P}}$.
\end{proof}

\subsection{Projectively Consistent Poisson Couplings on Standard Borel Spaces}\label{subsec:construction_2}

The preceding construction requires a single $\sigma$-finite measure dominating every distribution in the family under consideration. However, such a measure may not generally exist. For example, the uncountable family of all point masses over the real line admits no common $\sigma$-finite dominating measure. In fact, on a standard Borel space, a family $\P \subseteq \Delta(\Y)$ admits a common $\sigma$-finite dominating measure if and only if $\P$ is TV separable. Thus the core issue is not uncountability per se, but rather TV nonseparability. 
For any \emph{finite} tuple of probability measures $\mathbf{P}=(P_1,...P_k)\in\Delta(\Y)^k$, however, there is always a canonical finite dominating measure, $\nu_{\mathbf{P}}=\sum_i P_i$. Thus, for each fixed finite tuple, the Poisson construction \ref{cons:construction_1} can be applied with the tuple-dependent reference measure $\nu_{\mathbf{P}}$.

\begin{construction} \label{cons:construction_2}
    Fix $k\ge 1$ and a tuple $\mathbf{P}=(P_1,...P_k) \in \Delta(\Y)^k$. Define the $\mathbf{P}$-dependent measure $\nu_{\mathbf{P}} = \sum_i P_i$, which is finite, and let $f_i^{\nu_{\mathbf{P}}} = \frac{dP_i}{d\nu_{\mathbf{P}}}$ be the corresponding densities. Let $\A_{\mathbf{P}}$ be a Poisson point process on $\Y \times \R_+ \times \R_+$ with intensity $\nu_{\mathbf{P}}(dy) ds dt$ (Definition~\ref{def:PPP}). Unlike construction \ref{cons:construction_1}, the law of $A_{\mathbf{P}}$ depends on the entire tuple $\mathbf{P}$.
    
    For a realization $\mathrm{a}_{\mathbf{P}}$ of the Poisson point process $\A_{\mathbf{P}}$, we let $\Phi_i(\mathbf{P},\mathrm{a}_{\mathbf{P}})$ denote the $y$-coordinate of the point in $\mathrm{a}_{\mathbf{P},i}$ having the smallest $t$-coordinate, where the acceptance regions are $\mathrm{a}_{\mathbf{P}, i} =\{(y,s,t) \in \mathrm{a}_{\mathbf{P}}: s \le f_i^{\nu_{\mathbf{P}}}(y)\}$ for all $i \in [k]$.\footnote{Similar to \ref{cons:construction_1}, if $(y^*,s^*,t^*)$ is the unique point satisfying $t^* = \min\{t: (y,s,t) \in \mathrm{a}_{\mathbf{P}}, s \le f_i^{\nu_{\mathbf{P}}}(y)\}=\min\{t: (y,s,t)  \in \mathrm{a}_{\mathbf{P}, i}\}$, then $\Phi_i(\mathbf{P},\mathrm{a}_{\mathbf{P}})=y^*$. On the null event on which such a unique point does not exist, define $\Phi_i(\mathbf{P},\mathrm{a}_{\mathbf{P}})$ arbitrarily. Under the law of PPP, such events have probability zero, so it does not affect any of the distributions. See proof of Prop 5 in \cite{AngelS19}.}

    Using the same random process $\A_{\mathbf{P}}$ for every $i \in[k]$\footnote{It is worth noting that the same Poisson process $\A_{\mathbf{P}}$ is shared by all $k$ coordinates, and only the acceptance regions, $\A_{\mathbf{P}, i} =\{(y,s,t) \in \A_{\mathbf{P}}: s \le f_i^{\nu_{\mathbf{P}}}(y)\}$ differ across $i$}, we define the random variable $Z_{\mathbf{P},i} = \Phi_i(\mathbf{P},\A_{\mathbf{P}})$, and for every measurable $B\subseteq \Y^k$, define $\Gamma_k(\mathbf{P};B) = \Pr[(Z_{\mathbf{P},1},...,Z_{\mathbf{P},k}) \in B]$, i.e., $(Z_{\mathbf{P},1},\dots,Z_{\mathbf{P},k}) \sim \Gamma_k(\mathbf{P};\cdot)$. Thus, $\Gamma_k(\mathbf{P};\cdot)$ is the joint law induced by the $\mathbf{P}$-specific Poisson coupling.
\end{construction}

\begin{algorithm}
\caption{$\mathsf{Couple\_II}$  from construction~\ref{cons:construction_2}}
\label{alg:coupled-II}
\begin{algorithmic}[1]
\Require Distributions $q_1,\ldots,q_k \in \F$. 
\Ensure Synthetic samples $(Z_1,...,Z_k)$ such that $Z_i \sim q_i^\ell$ and $\Pr[Z_i \ne Z_j] \le 2 \TV(q_i^\ell, q_j^\ell)$ for all $i,j\in[k]$.
\State Set $P_i = q_i^\ell$ for all $i\in[k]$, $\nu_{\mathbf{P}}=\sum_i P_i$, and $f_i^{\nu_{\mathbf{P}}}=\frac{dP_i}{d \nu_{\mathbf{P}}}$
\State Let a random seed be $\A_{\mathbf{P}}$, a Poisson point process with intensity $\nu_{\mathbf{P}}(dy) ds dt$
\State $Z_i = \Phi_i(\mathbf{P}, \A_{\mathbf{P}})$ for all $i\in[k]$ according to construction~\ref{cons:construction_2}
\Comment{\textit{Conditional on a realization $\A_{\mathbf{P}}=\mathrm{a}_{\mathbf{P}}$, the output $z_i = \Phi_i(\mathbf{P}, \mathrm{a}_{\mathbf{P}})$ is deterministic.}}

\State \Return $(Z_1,...,Z_k)$.
\Comment{\textit{Equivalently, $(Z_1,...,Z_k)\sim \Gamma_k(\mathbf{P};\cdot)$ from construction~\ref{cons:construction_2}}}
\end{algorithmic}
\end{algorithm}

The key extra ingredients are to show that these tuple-specific couplings vary measurably with $\mathbf{P}$ and are projective consistent across different tuples. Measurability allows the construction to be applied to (data-dependent) random distributions, while projective consistency allows finite couplings built from different tuples to be embedded into a common larger coupling.

For a measurable map $T:\mathcal S\to\mathcal T$ and a
measure $\mu$ on $\mathcal S$, we write $T_\#\mu$ for the
pushforward measure of $\mu$ under $T$, defined by
$(T_\#\mu)(B)=\mu(T^{-1}(B))$ for all measurable
$B\subseteq\mathcal T$.
For finite products, write
$\mathrm{pr}_i(y_1,\ldots,y_k)=y_i$
for the $i$-th coordinate projection. 

\begin{proposition}[Projectively consistent Poisson coupling kernels]\label{prop:kernels}
Let $\mathcal Y$ be standard Borel. For every $k \ge 1$, Construction~\ref{cons:construction_2} defines a map
$\Gamma_k : \Delta(\mathcal Y)^k \to \Delta(\mathcal Y^k)$, with the following properties. Write
$\Gamma_k(\mathbf P;\cdot)$ for the measure $\Gamma_k(\mathbf P)$, and let
$D_{ij} = \{(y_1,\dots,y_k) \in \mathcal Y^k : y_i \ne y_j\}$.
\begin{enumerate}[label=\textup{(\arabic*)}]
  \item\label{it:meas} \emph{(Measurability.)} For every measurable
  $B \subseteq \mathcal Y^k$, the map
  $\mathbf P \mapsto \Gamma_k(\mathbf P; B)$ is measurable with respect to
  the evaluation $\sigma$-algebra on $\Delta(\mathcal Y)^k$; that is,
  $\Gamma_k$ is a Markov kernel from $\Delta(\mathcal Y)^k$ to $\mathcal Y^k$.
  \item\label{it:marg} \emph{(Marginals.)} For every
  $\mathbf P = (P_1,\dots,P_k)$ and $i \in [k]$,
  $(\mathrm{pr}_i)_{\#}\,\Gamma_k(\mathbf P;\cdot) = P_i$.
  \item\label{it:dis} \emph{(Disagreement bound.)} For every $\mathbf P$ and
  $i, j \in [k]$,
  $\Gamma_k(\mathbf P; D_{ij})
   \le \tfrac{2\,\mathrm{TV}(P_i,P_j)}{1+\mathrm{TV}(P_i,P_j)}
   \le 2\,\mathrm{TV}(P_i,P_j)$.
  \item\label{it:proj} \emph{(Projective consistency.)} For every nonempty
  $E = \{i_1, ..., i_{k'}\} \subseteq [k]$ with $i_1<...<i_{k'}$ and
  $\mathbf P' = (P_i)_{i \in E}$ in the original order,
  $(\mathrm{pr}_E)_{\#}\,\Gamma_k(\mathbf P;\cdot) = \Gamma_{k'}(\mathbf P';\cdot)$, where $\mathrm{pr}_E(y_1,...,y_k)=(y_{i_1}, ..., y_{i_{k'}})$
  \item\label{it:perm} \emph{(Permutation equivariance.)} For every
  permutation $\sigma$ of $[k]$: if
  $(Z_1,\dots,Z_k) \sim \Gamma_k\big((P_1,\dots,P_k);\cdot\big)$, then
  $\big(Z_{\sigma(1)},\dots,Z_{\sigma(k)}\big) \sim
   \Gamma_k\big((P_{\sigma(1)},\dots,P_{\sigma(k)});\cdot\big)$.
\end{enumerate}
\end{proposition}
\begin{proof} Now, we prove these properties.

\textbf{1. Measurability}
Fix $k\geq 1$ and let $\mathbf{P}=(P_1, \ldots,P_k) \in \Delta(\Y)^k$. Define $\nu_{\mathbf{P}} = \sum_i P_i$, so $P_i \ll \nu_{\mathbf{P}}$ for all $i\in[k]$. Since $\Y$ is standard Borel, the Radon-Nikodym derivatives may be chosen jointly measurably in $(\mathbf{P},y)$, so write $f_i(\mathbf{P},y)=\frac{dP_i}{d\nu_{\mathbf P}}(y).$ Since  $P_i \ll \nu_{\mathbf{P}}$, we may take $0\le f_i(\mathbf{P},y) \le 1$ $\nu_{\mathbf P}$-a.e.

For each $\mathbf{P}$, let $\A_{\mathbf{P}}$ be a Poisson point process on $\Y \times \R_+ \times \R_+$ with intensity $\nu_{\mathbf{P}}(dy) ds dt.$ For every $i\in[k]$, define
$\A_{\mathbf{P},i} = \{(y,s,t) \in \A_{\mathbf{P}}: s \le f_i(\mathbf{P},y) \}$.  We define the random variable $Z_i$ by $Z_i = y$, where $(y,s,t)$ is the point of $\A_{\mathbf{P}, i}$ with the minimal $t$ among all points of $\A_{\mathbf{P},i}$. Equivalently, we write $Z_i=\Phi_i(\mathbf{P},\A_{\mathbf{P}}),$ where $\Phi_i$ denotes this deterministic rule. 

Now define $\Phi(\mathbf{P},\A_{\mathbf{P}})=(\Phi_1(\mathbf{P},\A_{\mathbf{P}}), \ldots, \Phi_k(\mathbf{P},\A_{\mathbf{P}}))$ and for measurable $B \subseteq \Y^k$, let $\Gamma_k(\mathbf{P};B)=\Pr_{\A_{\mathbf{P}}}[\Phi(\mathbf{P},\A_{\mathbf{P}}) \in B]$, i.e. $\Gamma_k(\mathbf P;\cdot) = \mathcal{L}(\Phi(\mathbf{P},\A_{\mathbf{P}}))$. 
Since $\Y$ is standard Borel, the intensity measure $\nu_{\mathbf{P}}(dy) ds dt$ depends measurably on $\mathbf{P}$, and since the law of any PPP is uniquely and completely determined by its intensity measure, it follows that the $\mathbf{P}$-dependent PPP law varies measurably with $\mathbf{P}$. 
Also, $f_i(\mathbf{P}, y)$ is jointly measurable, so the indicator $1\{\{g(\mathbf{P}, y, s) \ge 0\}\}$ is also jointly measurable; applying a measurable filter to a point process yields another valid point process, and applying the argmin on $t$-coordinate is also measurable; and thus, the deterministic map $(\mathbf{P},a)\mapsto \Phi(\mathbf{P},a)$ is measurable. 
Therefore, $\mathbf{P} \mapsto \Gamma_k(\mathbf{P};B)$ is also measurable for every measurable $B\subseteq \Y^k$, so $\Gamma_k$ is a Markov kernel.

\textbf{2. Marginal}
To see that $Z_i$ has the required law (so that the above is indeed a coupling), we note
that $\A_{\mathbf{P},i}$ has intensity $f_i(\mathbf{P},y) \nu_{\mathbf{P}}(dy) dt = P_i(dy) dt$, so the $y$-coordinate of its first point, i.e. $\Phi_i(\mathbf{P},\A_{\mathbf{P}})$, has law $P_i$. Hence, $(\mathrm{pr}_i)_\# \Gamma_k(\mathbf{P};\cdot) = P_i$

\textbf{3. Disagreement bound}
Fix $i,j\in[k]$, and denote $g=f_i(\mathbf{P},y)$, $h=f_j(\mathbf{P},y)$, $\alpha = \TV(P_i, P_j).$ Consider the first point $(y, s)$ to appear that has $s \le g(y) \vee h(y)$. If it happens that $s \le g(y) \wedge h(y)$, then we get the agreement $Z_i = Z_j = y$; otherwise, this point determines either the value of $Z_i$ or $Z_j$. Hence, for any measurable $A$, the intensity of the region accepted by both is $\int_A (g \wedge h) d\nu_{\mathbf{P}}$ and the intensity of the region accepted by at least one is $\int_{\Y} (g \vee h) d\nu_{\mathbf{P}}$. Therefore, $\Pr[Z_i=Z_j] \ge \Pr[Z_i=Z_j \in A \text{ and the same point determines both $Z_i$ and $Z_j$}]=\frac{\int_A (g \wedge h) d\nu_{\mathbf{P}}}{\int_{\Y} (g \vee h) d\nu_{\mathbf{P}}}$. Note that $\int_{\Y} (g \vee h) d\nu_{\mathbf{P}} = 1+\alpha$ and $\int_{\Y} (g \wedge h) d\nu_{\mathbf{P}} = 1-\alpha$, so it follows that $\Pr[Z_i=Z_j] \ge \frac{1-\alpha}{1+\alpha}$, i.e. $\Pr[Z_i \ne Z_j] \le  1- \frac{1-\alpha}{1+\alpha} \le 2 \alpha.$ Thus, $\Gamma_k(\mathbf{P};D_{ij}) \le \frac{2\TV(P_i, P_j)}{1 + \TV(P_i, P_j)} \le 2\TV(P_i, P_j)$, for $D_{ij}=\{(y_1,...,y_k)\in \Y^k: y_i \ne y_j\}$.

\textbf{4. Projective consistency}
Let $E \subseteq [k]$ be nonempty. Let $\mathbf{P}'=(P_i)_{i\in E}$ and define $\nu_{\mathbf{P}'} = \sum_{i \in E} P_i$ and $f_i(\mathbf{P}',y) = \frac{d P_i }{d\nu_{\mathbf{P}'}}(y)$ for $i \in E$. Thus, by the chain rule, we have $f_i(\mathbf{P},y) = r(y) f_i(\mathbf{P}',y)$, where $r(y) = \frac{d\nu_{\mathbf{P}'} }{d\nu_{\mathbf{P}}}(y)$  $\nu_{\mathbf{P}}$-a.e, for all $i \in E$. 
Let $Y_+=\{y:r(y)>0\}$. Since $\nu_{\mathbf P'}(\Y\setminus Y_+)=0$ and $P_i\ll\nu_{\mathbf P'}$ for every $i\in E$, we also have $P_i(\Y\setminus Y_+)=0$. Equivalently, the portion of the acceptance region for coordinate $i$ lying over $\Y\setminus Y_+$ has zero intensity. Hence, almost surely, no point with $r(y)=0$ is accepted by any coordinate $i\in E$, and these points may be discarded.

Note that under the change of variable $(y,s,t) \mapsto (y, \frac{s}{r(y)}, t)$, the restriction of $\A_{\mathbf{P}}$ on $Y_+\times \R_+ \times \R_+$ is transformed into a PPP with intensity $\nu_{\mathbf{P}'}(dy) du dt$, for $u=s/r(y)$, since $\nu_{\mathbf{P}}(dy) ds dt=\nu_{\mathbf{P}}(dy) r(y) du dt=\nu_{\mathbf{P}'}(dy) du dt$. Moreover, for all $i \in E$, $s\leq f_i^{\mathbf{P}}(y)$ if and only if $\frac{s}{r(y)} \le f_i(\mathbf{P}', y) \iff u \le f_i(\mathbf{P}', y)$, so after rescaling the $s-$ coordinate, the transformation does not change the $y-$ and $t-$ coordinate, i.e. it preserves the first accepted point for every $i\in E$. In other words, if some point $(y,s,t)$ was the first accepted point for coordinate $i$ in the full construction, its transformed version $(y, s/r(y),t)$ has the same $y$ and $t$, and is accepted in the subtuple construction. Therefore, the joint law of the coordinates indexed by $E$ obtained from the full construction is exactly the law obtained by applying the Poisson construction directly to $\mathbf{P}'$. Hence, $(\mathrm{pr}_E)_\# \Gamma_k(\mathbf{P};\cdot) = \Gamma_{k'} (\mathbf{P}'; \cdot)$, for $k'=|E|$.

\textbf{5. Permutation equivariance}
Let $\mathbf{P}=(P_1,...,P_k)$ and denote $\mathbf{P}_\sigma = (P_{\sigma(1)},...,P_{\sigma(k)})$ for an arbitrary permutation $\sigma$ of $[k].$ Note that the dominating measure is invariant under permutations, i.e. $\nu_{\mathbf{P}}=\sum_i P_i = \sum_i P_{\sigma(i)} = \nu_{\mathbf{P}_\sigma}.$ Moreover, for every $i\in[k]$, $f_i(\mathbf{P}_\sigma,y) = \frac{d P_{\sigma(i)}}{ d \nu_{\mathbf{P}_\sigma} }= \frac{d P_{\sigma(i)}}{ d \nu_{\mathbf{P}}} = f_{\sigma(i)}(\mathbf{P},\cdot)$ $\nu_{\mathbf P}$-a.e. Also note that a Poisson point process $\A_{\mathbf{P}}$ on $\Y \times \R_+ \times \R_+$ with intensity $\nu_{\mathbf{P}}(dy) ds dt$ is also the Poisson point process used in the construction for
$\mathbf{P}_{\sigma}$, since $\nu_{\mathbf{P}}=\nu_{\mathbf{P}_\sigma}.$ Thus, the acceptance regions are the same a.s. as well, i.e. $\A_{\mathbf{P}_\sigma,i} = \{(y,s,t) \in \A_{\mathbf{P}}: s \le f_i(\mathbf{P}_\sigma,y) \} =\{(y,s,t) \in \A_{\mathbf{P}}: s \le f_{\sigma(i)}(\mathbf{P},y) \} =  \A_{\mathbf{P},\sigma(i)}$. Hence, if we similarly define $Z_{\sigma, i} = \Phi_i(\mathbf{P}_\sigma, \A_{\mathbf{P}})$, then we have $Z_{\sigma, i} = \Phi_i(\mathbf{P}_\sigma, \A_{\mathbf{P}})= \Phi_{\sigma(i)}(\mathbf{P}, \A_{\mathbf{P}}) = Z_{\sigma(i)}$, where $Z_i = \Phi_i(\mathbf{P}, \A_{\mathbf{P}})$. Thus, $(Z_{\sigma, 1}, ..., Z_{\sigma, k}) \sim \Gamma_k(\mathbf{P}_\sigma;\cdot)$
\end{proof}

There are two aspects of \ref{prop:kernels} worth remarking.
First, the standard Borel machinery is not needed to construct the coupling for a fixed tuple $\mathbf{P}$. In fact, for fixed $\mathbf{P}$, the Construction~\ref{cons:construction_2} is simply the globally-dominated construction as in Construction~\ref{cons:construction_1} using the $\mathbf{P}$-dependent dominating measure $\nu_{\mathbf{P}}=\sum_i P_i$. The standard-Borel assumption is used to ensure that the relevant Radon-Nikodym derivatives and Poisson law can be parameterized measurably in $\mathbf{P}$ so that $\Gamma_k$ is indeed a Markov kernel. This is what allows us later to sample from $\Gamma_k(\mathbf{P};\cdot)$ when $\mathbf{P}$ itself is a random data-dependent tuple of learner outputs.

Second, projective consistency allows couplings laws associated with different tuples to be realized as marginals of a single larger coupling. For example, given $\mathbf{P}=(P_1,...,P_k)$ and $\mathbf{Q}=(Q_1,...,Q_k)$, and consider $\Gamma_{2k}((P_1,...,P_k,Q_1,...,Q_k));\cdot)$. Projective consistency guarantees that its first $k$ coordinates have law $\Gamma_k(\mathbf{P};\cdot)$, and its last $k$ coordinates have law $\Gamma_k(\mathbf{Q};\cdot)$. Thus, the larger coupling jointly realizes the two coupling laws without changing either of their original joint distributions. This is the key property that we will need for privacy, where we couple two neighboring executions, and utility analyses, where we enlarge the current execution by adjoining the ground-truth law.

It is also worth noting the differences between the two constructions. Under a common $\sigma$-finite dominating measure $\nu$, construction \ref{cons:construction_1} produces a single coupling $(Z_P)_{P \in \mathcal{P}}$ for arbitrary (potentially uncountable) family $\mathcal{P} \subseteq \Delta(\Y)$, using one single global Poisson point process. On the other hand, construction \ref{cons:construction_2} assumes \emph{no global dominating measure}, and for each finite tuple $\mathbf{P}$, constructs a tuple-specific coupling using the canonical dominating measure $\nu_{\mathbf{P}}$ that is also tuple-dependent. Although for a different finite tuple $\mathbf{Q}$, the constructions generally uses a different reference measure $\nu_{\mathbf{Q}}$ and hence a different PPP law and coupling, the projection consistency allows us to create a larger coupling of the concatenated tuple, such that the coordinates corresponding to $\mathbf{Q}$ have law $\Gamma_k(\mathbf{Q};\cdot)$, and the rest have law $\Gamma_k(\mathbf{P};\cdot)$

Despite this bottom-up approach, we can still establish a post-hoc universal coupling over all of $\Delta(\mathcal{Y})$ via the Kolmogorov Extension Theorem \cite{Kallenberg02}. Specifically, the projective consistency established in Proposition~\ref{prop:kernels} allows us to stitch the entire family of finite-dimensional couplings $(\Gamma_k)_{k \ge 1}$ together into a single, global probability measure over the uncountable product space $\mathcal{Y}^{\Delta(\mathcal{Y})}$, yielding a simultaneous coupling for every possible distribution in $\Delta(\mathcal{Y})$.
Consequently, the conceptual takeaway for both constructions is identical. In Construction~\ref{cons:construction_1}, one conceptually draws the family $(Z_P)_{P\in\mathcal{P}}$ from the joint law $\Gamma$. In Construction~\ref{cons:construction_2}, one conceptually draws the entire universe $(Z_P)_{P\in\Delta(\mathcal{Y})}$ from $\Gamma_{\infty}$, and trivially restricts to the coordinates $(Z_P)_{P\in\mathcal{P}}$ for any subset $\mathcal{P}$ of interest. The only caveat lies in how the (conceptual) couplings are constructed. In Construction~\ref{cons:construction_1}, we assume a dominating measure $\nu$ such that $P \ll \nu$ for all $P \in \mathcal{P}$, and construct the coupling directly based on a single PPP $\mathcal{A}$ depending only on $\nu$. In Construction~\ref{cons:construction_2}, we instead build a $\mathbf{P}$-dependent measure $\nu_{\mathbf{P}}$ and a corresponding PPP $\mathcal{A}_{\mathbf{P}}$ to define a coupling $\Gamma_k(\mathbf{P};\cdot)$ for each finite tuple $\mathbf{P}$, and then KET guarantees that this family of finite-dimensional couplings $(\Gamma_k)_{k \ge 1}$ ultimately induces a global coupling $\Gamma_{\infty}$ over the entire space $\Delta(\mathcal{Y})$.

\begin{corollary}\label{cor:ket_coupling}
In the setup above, let $(\Gamma_k)_{k\ge1}$ be the family of coupling kernels from Proposition~\ref{prop:kernels}. Then there exists a probability measure $\Gamma_\infty$ on $\Y^{\Delta(\Y)}$ equipped with the product $\sigma$-algebra such that if $(Z_P)_{P \in\Delta(\Y)} \sim \Gamma_\infty$, then for every $k\ge 1$ and every finite collection of probability measures $\mathbf{P}=(P_1,...,P_k)$, $(Z_{P_1},...,Z_{P_k}) \sim \Gamma_k(\mathbf{P};\cdot)$. In particular, for every $P\in \Delta(\Y)$, $Z_P \sim P$ and for every $P,Q \in \Delta(\Y)$, $\Pr[Z_P \ne Z_Q] \le 2 \TV(P,Q)$. In other words, $\Gamma_\infty$ is a grand coupling of all probability measure in $\Delta(\Y)$, i.e. $\Gamma_\infty$ is fixed by the family $(\Gamma_k)_k$ and does not depend on any particular finite tuple.
\end{corollary}
\begin{proof}
Consider index set $T=\Delta(\Y)$ and the coordinate spaces $S_P=\Y$ for $P\in T$. For every finite subset $I=\{P_1,...,P_k\} \subseteq T$, we choose an arbitrary ordering $(P_1,...,P_k)$. Under the natural isomorphism, $(y_P)_{P\in I} \rightarrow (y_{P_i})_{i \in [k]}$, we define $\mu_I$ to be the probability measure on $S_I = \prod_{P_i \in I} \Y = \Y^I$, corresponding to $\Gamma_k((P_1,...,P_k);\cdot)$. Note that such definition is independent of the chosen ordering, since for any other ordering of the form $(P_{\sigma(1)},...,P_{\sigma(k)})$ with arbitrary permutation $\sigma$ of $[k]$, Proposition~\ref{prop:kernels} shows that $\big(Z_{\sigma(1)},\dots,Z_{\sigma(k)}\big) \sim\Gamma_k\big((P_{\sigma(1)},\dots,P_{\sigma(k)});\cdot\big)$ if $(Z_1,\dots,Z_k) \sim \Gamma_k\big((P_1,\dots,P_k);\cdot\big)$. Hence, the two orderings induce the same law on the coordinate-indexed space $\Y^I.$

Now let $J \subseteq I$. Using the same ordering, projective consistency in Proposition~\ref{prop:kernels} implies $(\mathrm{pr}_{J})_\# \mu_I=\mu_J$, where $ \mathrm{pr}_{J}: \Y^I \mapsto \Y^J$ is the coordinate projection. Hence, $(\mu_I)_I$ is a projective family, and thus, the Kolmogorov Extension Theorem~\cite{Kallenberg02} yields random variables $(Z_P)_{P \in \Delta(\Y)}$ with $Z_P \in \Y$ such that for every finite collection of distinct measures $I=\{P_1,...,P_k\}$, $(Z_{P_1}, ..., Z_{P_k}) \sim \Gamma_k((P_1,...,P_k);\cdot)$. Let $\Gamma_\infty = \mathrm{Law}((Z_P)_{P \in \Delta(\Y)})$, which is a probability measure on $\Y^{\Delta(\Y)}$ equipped with its product $\sigma$-algebra.

It remains to remove the distinctness assumption. Fix an arbitrary tuple $\mathbf{P}=(P_1,...,P_k)\in T^k$ and let $Q_1,...,Q_n$ be the distinct measures (in the same ordering) in $\mathbf{P}$, where $P_i = Q_{\pi(i)}$ for some map $\pi : [k] \mapsto [n]$. Now define a duplication enlargement map $T_\pi: \Y^n \mapsto \Y^k$, where $T_\pi(y_1,...,y_n)=(y_{\pi(1)},...,y_{\pi(k)})$.  Since $Q$'s are distinct, it follows that $(Z_{Q_1}, ..., Z_{Q_n}) \sim \Gamma_n((Q_1,...,Q_n);\cdot)$. Moreover, since $(Z_P)_{P\in \Delta(\Y)}\sim \Gamma_\infty$ is indexed by the probability measures themselves, $P_i=Q_{\pi(i)}$ implies that $Z_{P_i}=Z_{Q_{\pi(i)}}$ is the same coordinate of the grand coupling, so it follows that $(Z_{P_1},\ldots,Z_{P_k})=T_\pi(Z_{Q_1},...,Z_{Q_n})$ pointwise. Therefore, $\mathrm{Law}_{\Gamma_{\infty}}(Z_{P_1},...,Z_{P_k})=(T_\pi)_\# \Gamma_n((Q_1,...,Q_n);\cdot).$ 

Next, it suffices to show that $(T_\pi)_\# \Gamma_n((Q_1,...,Q_n);\cdot) = \Gamma_k(\mathbf P;\cdot)$. Now consider $E=\{r_1,...,r_n\}$ where $r_j=\min\{i\in[k]: P_i = Q_j\}$ is the first representative coordinate for $Q_j$ in $\mathbf{P}$. By projection consistency, we have $(\mathrm{pr}_E)_\# \Gamma_k(\mathbf P;\cdot) = \Gamma_n((Q_1,...,Q_n);\cdot)$ since $Q$'s are distinct. Also note that $\Gamma_k(\mathbf P;D_{ij}) \le 2\TV(P_i,P_j) =0$ if $P_i = P_j$, so for $(Y_1,...,Y_k)\sim \Gamma_k(\mathbf P; \cdot)$, $Y_i = Y_{r_{\pi(i)}}$ a.s., and thus $(Y_1,...,Y_k)=T_\pi(Y_{r_1},...,Y_{r_n})$ a.s. Hence, $\Gamma_k(\mathbf P; \cdot)=(T_\pi)_\# \mathrm{Law}(Y_{r_1},...,Y_{r_n}) = (T_\pi)_\# \Gamma_n((Q_1,...,Q_n);\cdot) = \mathrm{Law}_{\Gamma_{\infty}}(Z_{P_1},...,Z_{P_k})$. Thus, we have shown that for every $k\ge1$ and every tuple $\mathbf{P}\in\Delta(\Y)^k$ that could contain repeated measures, $(Z_{P_1},...,Z_{P_k}) \sim \Gamma_k(\mathbf{P};\cdot)$.

Finally, for all $P\in\Delta(\Y)$, taking $I= \{P\}$ yields $Z_P \sim \Gamma_1(P;\cdot)$, and by Proposition~\ref{prop:kernels}, $\Gamma_1(P;\cdot)=P$, so $Z_P \sim P$. For all $P, Q\in\Delta(\Y)$, taking $I= \{P,Q\}$ yields $(Z_P,Z_Q) \sim \Gamma_2((P,Q);\cdot)$, so $\Pr[Z_P\ne Z_Q]=\Gamma_2((P,Q); D_{12})$ for $D_{12}=\{(y_1,y_2)\in \Y^2: y_1 \ne y_2 \}$. By the disagreement bound in Proposition~\ref{prop:kernels}, it follows that $\Pr[Z_P\ne Z_Q]=\Gamma_2((P,Q); D_{12})\le 2\TV(P,Q)$.
\end{proof}

Having established the existence of the global measure $\Gamma_\infty$, we now formally confirm that it constitutes a grand coupling scheme on $\Delta(\mathcal{Y})$, and trivially after restriction, on any subset $\mathcal{P}$.

\begin{corollary}[Grand coupling scheme $\Gamma_{\mathcal{P}}$ from Construction~\ref{cons:construction_2}]\label{cor:cons2_grand_coupling}
The family of random variables $(Z_P)_{P \in\Delta(\mathcal{Y})} \sim \Gamma_\infty$ constructed in Corollary~\ref{cor:ket_coupling} constitutes a universal grand coupling scheme on the entire space $\Delta(\Y)$. Consequently, its trivial restriction $(Z_P)_{P \in \mathcal{P}}$ yields a valid grand coupling scheme $\Gamma_{\P}$ on any subset $\P \subseteq \Delta(\Y)$.
\end{corollary}
\begin{proof}
G1 and G2 directly follows from Corollary~\ref{cor:ket_coupling}, and G3 follows from Proposition~\ref{prop:kernels}. The restriction to any subset $\mathcal{P}$ preserves all three properties trivially.Thus, the joint law of $(Z_P)_{P \in \mathcal{P}}$ yields the desired grand coupling scheme $\Gamma_{\mathcal{P}}$.
\end{proof}

\section{Private synthetic data sampler}\label{app:sampler-details}

In this section, we apply the grand coupling schemes developed in Section~\ref{sec:coupling} to  construct a differentially private synthetic data sampler. In particular, we formally prove the privacy and utility guarantees of \textsc{Aggregate} (Algorithm~\ref{alg:coupled-sampler}) and \textsc{Sampler} (Algorithm~\ref{alg:private-correlated-sampler}), presented in Section~\ref{sec:private-data-sampler}.

Throughout, let $\Y=\X^\ell$, where $\mathcal{X}$ is the data universe and $\ell$ is the size of the synthetic dataset, and the space of probability measures of interest is
$\P=\{q^\ell:q\in\F\}$, where $\F$ contains the ground-truth distribution class $\Q$ and all possible outputs of the base learner.

\subsection{\textsc{Aggregate}: Tuple-Level Privacy and Utility}

We establish the privacy and utility guarantees of \textsc{Aggregate} (Algorithm~\ref{alg:coupled-sampler}), instantiated with any grand coupling scheme on $\P$ satisfying Definition~\ref{def:grand-scheme}.

\paragraph{Tuple-Level Privacy.}
Recall that we use a change-by-one adjacency notion for neighboring tuples, i.e., $\mathbf{q}=(q_1,\dots,q_k)$ and $\mathbf{q}'=(q_1',\dots,q_k')$ are neighboring if they differ in at most one coordinate.

The privacy argument relies on using the grand coupling scheme $\Gamma_{\P}$ to jointly evaluate the algorithm on both $\q$ and $\q'$. Since the joint law of $\mathsf{Couple}_k(\q)$, and likewise for $\q'$, in $\textsc{Aggregate}$ (Algorithm~\ref{alg:coupled-sampler}) is simply a finite-dimensional marginal of the global measure $\Gamma_{\mathcal{P}}$\footnote{For $\mathsf{Couple\_I}$ (Construction~\ref{cons:construction_1}), the joint law of $(Z_1,\dots,Z_k)\sim \mathcal{L}(\mathsf{Couple}_k(\q))$ is the $(q_1^\ell,\dots,q_k^\ell)$-coordinate marginal of the explicit coupling $\Gamma$ on $\mathcal{P}=\{q^\ell : q \in \mathcal{F}\}$. For $\mathsf{Couple\_II}$ (Construction~\ref{cons:construction_2}), it is the corresponding finite-dimensional marginal of the grand coupling $\Gamma_\infty$ on $\mathcal{Y}^{\Delta(\mathcal{Y})}$, where $\mathcal{Y}=\mathcal{X}^\ell$.}, we can evaluate both executions on a single probability space.

Specifically, rather than drawing independent samples from $\mathcal{L}(\mathsf{Couple}_k(\q))$ and $\mathcal{L}(\mathsf{Couple}_k(\q'))$, respectively, we construct an analysis coupling from a single draw $(Z_P)_{P \in \mathcal{P}} \sim \Gamma_{\P}$, setting $\tilde{Z}_i = Z_{q_i^\ell}$ for input $\q$ and $\tilde{Z}'_i = Z_{(q_i')^\ell}$ for input $\q'$. For every unchanged coordinate $i \neq j$, $q_i = q_i'$ implies $Z_{q_i^\ell} = Z_{(q_i')^\ell}$ pointwise. Consequently, the realized input vectors $(\tilde{Z}_1, \dots, \tilde{Z}_k)$ and $(\tilde{Z}'_1, \dots, \tilde{Z}'_k)$ differ in at most a single coordinate $j$ for \emph{every realization} of $\Gamma_{\P}$. Therefore, since neighboring input tuples yield pointwise neighboring vectors under $\Gamma_{\P}$, the $(\varepsilon,\delta)$-DP guarantee of $\mathsf{StabilityBasedMode}$ transfers directly to $\textsc{Aggregate}$.

\begin{proposition}[Privacy of \textsc{Aggregate}]
\label{prop:sampler-privacy}
If $\mathsf{StabilityBasedMode}$ is $(\varepsilon,\delta)$-DP, then \textsc{Aggregate} in Algorithm~\ref{alg:coupled-sampler} is $(\varepsilon,\delta)$-DP.\footnote{For $\mathsf{StabilityBasedMode}$, adjacency is defined as replacing a single dataset $Z_i$ in $(Z_1,\dots,Z_k)$; for $\textsc{Aggregate}$, adjacency is defined as replacing a single distribution $q_i$ in $\mathbf{q}=(q_1,\dots,q_k)$.}
\end{proposition}
\begin{proof}
    Fix a pair of neighboring tuples $\q=(q_1,...,q_k)$ and $\q'=(q_1',...,q_k')$, where $q_i = q_i'$ for all $i\ne j$. Let $P_i = q_i^\ell$ and $P_i' = q_i'^{\ell}$. Hence, $P_i = P_i'$ for all $i\ne j$. We denote $\mathbf{P}=(P_1,...,P_k)$ and $\mathbf{P}'=(P_1',...,P_k')$ 
    
    By Definition~\ref{def:grand-scheme}, let $(Z_P)_{P \in \P}$ be the grand coupling scheme on $\P$ used in $\mathsf{Couple}_k$, and, on the common probability space where $(Z_P)_{P \in \mathcal{P}}$ is defined, we construct the vectors $\tilde{\mathbf{Z}} = (Z_{P_1}, \dots, Z_{P_k})$ and $\tilde{\mathbf{Z}}' = (Z_{P_1'}, \dots, Z_{P_k'})$.
    Note that the joint law of the output of $\mathsf{Couple}_k(\mathbf{q})$ is exactly the marginal law $\mathcal{L}(Z_{P_1}, \dots, Z_{P_k})$, and similarly for $\q'$, by property G3, so it follows that $\tilde{\mathbf{Z}} \overset{\text{d}}{=} \mathsf{Couple}_k(\mathbf{q})$ and $\tilde{\mathbf{Z}}' \overset{\text{d}}{=} \mathsf{Couple}_k(\mathbf{q}')$.

    Furthermore, since $P_i = P_i'$ for all $i \ne j$, we have $\tilde{Z}_i = \tilde{Z}_i'$ pointwise for all $i \ne j$. Thus, the vectors $(\tilde{Z}_1,\dots,\tilde{Z}_k)$ and $(\tilde{Z}_1',\dots,\tilde{Z}_k')$ are neighboring.
    Let $M(\mathbf{Z};B)$ denote the probability that $\mathsf{StabilityBasedMode}$ outputs an element of a measurable set $B$ on input $\mathbf{Z}$. Since $M$ is $(\varepsilon,\delta)$-DP and our coupled input vectors are adjacent a.s., $M((\tilde Z_1,\dots,\tilde Z_k);B) \le e^\varepsilon M((\tilde Z_1',\dots,\tilde Z_k');B) + \delta$. Taking the expectation over the global measure $\Gamma$ yields:
    \begin{align*}
        \Pr[\textsc{Aggregate}(\mathbf{q}) \in B]&=\E_{\Gamma}[M((\tilde Z_1,\dots,\tilde Z_k);B)] \\
        &\leq e^\varepsilon \E_{\Gamma}[M((\tilde Z_1',\dots,\tilde Z_k');B)] +\delta \\
        &= e^\varepsilon \Pr[\textsc{Aggregate}(\mathbf{q}') \in B]  +\delta.
    \end{align*}
    A symmetric argument holds for the reverse direction.
\end{proof}

\paragraph{Tuple-Level Utility}

We next show that if a large majority of the input distributions $q_i$ are close to a target ground-truth distribution $p$ in TV, then the output of $\textsc{Aggregate}$ matches an ideal draw $p^\ell \sim p^\ell$ with high probability, and is thus close to $p^\ell$ in TV.

As in the privacy analysis, the utility guarantee relies on constructing an analysis coupling from the grand coupling scheme $\Gamma_{\P}$. Specifically, for input distributions $\q = (q_1, \dots, q_k) \in \F^k$ and a target distribution $p \in \F$, we take a single draw $(Z_P)_{P \in \P} \sim \Gamma_{\P}$ and define $\tilde{Z}_i = Z_{q_i^\ell}$ for each $i \in [k]$, alongside an unobserved ``ghost'' sample $\tilde{Z}_* = Z_{p^\ell}$, indexed by the ground-truth product distribution $p^\ell \in \P$. Since $\mathcal{L}(\mathsf{Couple}_k(\mathbf{q}))$ and $\mathcal{L}(\mathsf{Couple}_1(\mathbf{p}))$ are simply a finite-dimensional marginal of the global measure $\Gamma_{\mathcal{P}}$, so our construction from the grand coupling scheme $\Gamma_{\P}$ satisfies $(\tilde{Z}_1,...,\tilde{Z}_k) \sim \mathcal{L}(\mathsf{Couple}_k(\mathbf{q}))$ and $\tilde{Z}_* \sim \mathcal{L}(\mathsf{Couple}_1(p))=p^\ell$, by property G3 and G1.  

We emphasize that $\tilde{Z}_*$ is never computed by the algorithm, as it is adjoined purely for the analysis. For every accurate distribution $q_i$ where $\mathrm{TV}(q_i, p) \le \eta$, property G2 guarantees that $\tilde{Z}_i = \tilde{Z}_*$ except with probability at most $2\ell\eta$. Consequently, when $7/8$ of the $q_i$ are $\eta$-close to $p$, an overwhelming majority of the coupled samples $\tilde{Z}_i$ collapse onto the exact same value $\tilde{Z}_* \sim p^\ell$ pointwise, creating a large frequency gap required for $\mathsf{StabilityBasedMode}$ to succeed.

\begin{proposition}[Utility of \textsc{Aggregate}]\label{prop:coupled-sampler-utility}
Fix $p\in\Q \subseteq \F$ and a tuple $\q=(q_1,...,q_k) \in \F^k$. Denote $J = \{j \in [k]: \TV(q_j,p)\le \eta\}$ for some fixed $\eta$. Suppose $\mathsf{StabilityBasedMode}$ outputs the mode with probability at least $1-\beta_{\mathsf{SBM}}$, if the frequency gap (between the mode and the second most frequently occurring element) is at least $\tau$. Then, for every $c \in (0,1)$ such that $2(1-c) |J| - k \ge \tau$, 
$\TV(\mathcal{L}(\textsc{Aggregate}(\q)), p^\ell) \le \frac{2\ell \eta}{c} + \beta_{\mathsf{SBM}}$
\end{proposition}
\begin{proof}
    Fix a tuple $\q=(q_1,...,q_k) \in \F^k$ and a ground truth $p\in\Q$. Let $P_i = q_i^\ell$ and $P_* = p^{\ell}$.
    We construct an analysis coupling using a single draw $(Z_P)_{P \in \mathcal{P}} \sim \Gamma_{\mathcal{P}}$ from our grand coupling scheme. We evaluate the coordinates corresponding to the input distributions, setting $\tilde{Z}_i = Z_{P_i}$ for $i \in [k]$, alongside an unobserved ``ghost'' sample $\tilde{Z}_* = Z_{P_*}$. By properties (G1) and (G3), $(\tilde{Z}_1, \dots, \tilde{Z}_k) \sim \mathcal{L}(\mathsf{Couple}_k(\mathbf{q}))$ and $\tilde{Z}_* \sim p^\ell$.
    Similar to the proof of Proposition~\ref{prop:sampler-privacy}, we construct an analysis coupling using a single draw $(Z_P)_{P \in \mathcal{P}} \sim \Gamma_{\mathcal{P}}$ from our grand coupling scheme. Specifically, we consider the coordinates corresponding to the input distributions, setting $\tilde{Z}_i = Z_{P_i}$ for $i \in [k]$, alongside an unobserved ``ghost'' sample $\tilde{Z}_* = Z_{P_*}$. By properties G1 and G3, $(\tilde{Z}_1, \dots, \tilde{Z}_k) \sim \mathcal{L}(\mathsf{Couple}_k(\mathbf{q}))$ and $\tilde{Z}_* \sim \mathsf{Couple}_1(p) = p^\ell$.

    By the disagreement bound (G2), we have $\Pr[\tilde Z_i \ne \tilde Z_*] \le 2\TV(P_i, P_*)=2\TV(q_i^\ell, p^\ell)$ for all $i \in [k]$. For indices in our accurate set $J=\{j \in [k]: \TV(q_j,p)\le \eta\}$, this implies $\Pr[\tilde Z_j \ne \tilde Z_*] \le 2\ell \eta$.

    Consider the indicator $I_j = 1\{\tilde Z_j \neq \tilde Z_*\}$ for all $j\in J$. 
    We have that $\E [\sum_{j \in J} I_j] \leq 2\ell \eta |J|$, and by Markov's inequality, we have $\Pr[\sum_{j \in J} I_j > c |J|] \le \frac{2\ell \eta}{c}$. On the event $\sum_{j \in J} I_j \le c |J|$, there are at least $|J|-\sum_{j \in J} I_j \ge (1-c) |J|$ coordinates of $(\tilde Z_1,...,\tilde Z_k)$ equal to $\tilde Z_*$.
    Since $2(1-c)|J| - k \ge \tau > 0$, by our choice of $c$, it follows that $(1-c)|J| > k/2$, ensuring $\tilde{Z}_*$ appears in a strict majority of the $k$ coordinates. Consequently, $\tilde{Z}_*$ is the unique mode, any competing element appears at most $k - (1-c)|J|$ times, and the frequency gap satisfies, $2(1-c) |J| - k \ge \tau$. 

    Let $\tilde Z$ be the output obtained by applying $\mathsf{StabilityBasedMode}$ on the analysis-coupled vector $(\tilde Z_1,...,\tilde Z_k)$, as in Algorithm~\ref{alg:coupled-sampler}. Since $(\tilde{Z}_1,...,\tilde{Z}_k) \overset{\text{d}}{=} \mathsf{Couple}_k(\q)$, it follows that $\tilde Z \overset{\text{d}}{=} \textsc{Aggregate}(\q)$. Hence, $(\tilde Z, \tilde Z_*)$ is indeed a coupling of $\mathcal{L}(\textsc{Aggregate}(\q))$ and $p^\ell$.
    
    By the accuracy guarantee of $\mathsf{StabilityBasedMode}$, we have $\Pr[\tilde Z \ne \tilde Z_* | \sum_{j \in J} I_j \le c |J|] \leq \beta_{\mathsf{SBM}}$. Thus, $\Pr[\tilde Z \ne \tilde Z_* ] \leq \beta_{\mathsf{SBM}} + \frac{2\ell \eta}{c}$. Additionally, $\TV(\mathcal{L}(\textsc{Aggregate}(\q)), p^\ell) = \TV(\mathcal{L}(\tilde Z), \mathcal{L}(\tilde Z_*)) \leq \Pr[\tilde Z \ne \tilde Z_* ] \leq \beta_{\mathsf{SBM}} + \frac{2\ell \eta}{c}$.
\end{proof}

To complete the utility guarantee, we instantiate the mode mechanism step with the stability-based mode mechanism \cite{ts13, Vadhan17}, whose privacy and accuracy properties are recalled below.

\begin{fact}[Stability-based mode, Proposition 3.4 from \cite{Vadhan17}; \cite{ts13}]\label{fact:mode}
For every $\epsilon, \delta \in (0,1)$ and $m \in \mathbb N$, there exists an $(\varepsilon, \delta)$-DP mechanism $M:\mathcal X^m \to \mathcal X$ such that
if the difference in counts between the mode of $\mathcal X^m$ and the second most frequently occurring element is larger than $4\log(1/\delta)/\varepsilon$, $M(x)$ outputs the mode with probability at least $1 -\delta$.
\end{fact}

We can now synthesize these components. By combining our tuple-level privacy guarantee (in Proposition ~\ref{prop:sampler-privacy}), the utility bound of our analysis coupling (Proposition~\ref{prop:coupled-sampler-utility}), and the accuracy of the mode mechanism (Fact~\ref{fact:mode}), we obtain the final privacy and utility guarantee for the \textsc{Aggregate} subroutine, stated below.

\begin{corollary}\label{cor:instantiated-coupled-sampler-utility}
Let $\beta,\varepsilon,\delta\in (0,1)$, $\ell\in \mathbb N^+$ be parameters. Set $\eta = \frac{\beta}{64\ell}$, $\delta_0=\min\{\delta,\beta/4\}$, and choose $k$ such that $k \geq 8  \frac{\log(4/\beta)+\log(1/\delta)}{\varepsilon}$. Then \textsc{Aggregate}, with $\mathsf{StabilityBasedMode}$
instantiated using parameters $(\eps,\delta_0)$ as in
Fact~\ref{fact:mode}, is $(\eps,\delta_0)$-DP.
Moreover, for any $p\in \Q$ and input tuple $\mathbf{q}\in \F^k$, if $|\{j \in [k]: \TV(q_j,p)\le \eta\}| \geq 7k/8$, then $\TV(\mathcal{L}(\textsc{Aggregate}(\q)), p^\ell) \le \beta/2$.
\end{corollary}
\begin{proof}
By Proposition~\ref{prop:sampler-privacy}, \textsc{Aggregate} is $(\varepsilon,\delta_0)$-DP, since $\mathsf{StabilityBasedMode}$ is $(\varepsilon,\delta_0)$-DP.
For utility, take $\tau = 4\log(1/\delta_0)/\varepsilon$ and $\beta_{\mathsf{SBM}}=\delta_0$ from Fact~\ref{fact:mode}, and choose $c=1/8$. It is not hard to verify that $2(1-c)|J|-k \ge k/2 \ge \tau$. Proposition~\ref{prop:coupled-sampler-utility} therefore gives a coupling $(\tilde Z, \tilde Z_*)$ such that $\tilde Z \overset{\text{d}}{=} \textsc{Aggregate}(\q)$, $\tilde Z_* \overset{\text{d}}{=} p^\ell$, and $\Pr[\tilde Z \ne \tilde Z_* ] \leq \beta_{\mathsf{SBM}} + \frac{2\ell \eta}{c} = \delta_0 + \beta/4 \le \beta / 2$. Consequently, $\TV(\mathcal{L}(\textsc{Aggregate}(\q)), p^\ell) \le \beta/2$
\end{proof}

\subsection{\textsc{Sampler}: End-to-End Private Correlated Sampling from Samples}

We establish the privacy and utility guarantees of \textsc{Sampler} (Algorithm~\ref{alg:private-correlated-sampler}), which takes a dataset $S\in\X^n$ as input. 
Here, two datasets $S,S'\in\X^n$ are adjacent if they differ in at most one record.

\begin{proposition}[Privacy of \textsc{Sampler}]\label{prop:PCS-privacy}
If $\mathsf{StabilityBasedMode}$ in \textsc{Aggregate} is $(\varepsilon,\delta)$-DP, then
\textsc{Sampler} is $(\varepsilon,\delta)$-DP. \footnote{For $\mathsf{StabilityBasedMode}$, adjacency is defined as replacing a single dataset $Z_i$ in $(Z_1,\dots,Z_k)$; adjacency is defined as replacing a single record $X_j$ in the raw dataset $S=(X_1,\dots,X_n) \in \mathcal{X}^n$}
\end{proposition}
\begin{proof}
    Fix neighboring datasets $S=(X_1,...,X_n)$ and $S'=(X_1',...,X_n')$ such that $X_i = X_i'$ for all $i\ne j$.

    We first couple the data-independent randomness used in the two executions on $S$ and $S'$. Specifically, we draw a single permutation $\Pi \sim \mathrm{U}(\mathcal{S}_n)$ and a single collection of learner seeds $U_1,...,U_k \overset{\text{i.i.d.}}{\sim} \mathrm{U}([0,1])$ independently of data, and use the same realizations of them in both executions.

    Now fix a realization $(\pi, u_1,...,u_k)$. Denote the partitions as $(s_i)_i$ and $(s_i')_i$ on the two executions.
    Let $a$ be the block id containing the position $\pi^{-1}(j)$, i.e. the position occupied by the differing data coordinate under permutation. Since $X_i=X_i'$ for all $i\ne j$, it follows that $s_i = s_i'$ for all $i\ne a$. Therefore, with the same learner seed $u_i$ used on $s_i$ and $s_i'$, it follows that $q_i= L(s_i;u_i) = L(s_i';u_i)=q_i'$ for all $i\ne a$. Hence, for every fixed realization of the data-independent randomness, the induced tuples $\q=(q_1,...,q_k)$ and $\q'=(q_1',...,q_k')$ differ in at most one coordinate.

    Let $M(\q;B)$ denote the probability that \textsc{Aggregate}$(\q)$ lies in a measurable set $B \subseteq \X^\ell$. Since $\mathsf{StabilityBasedMode}$ is $(\varepsilon,\delta)$-DP, by Proposition~\ref{prop:sampler-privacy}, \textsc{Aggregate} is $(\varepsilon,\delta)$-DP under change-by-one adjacency on input tuple. Hence, for every realization $(\pi,u_1,\ldots,u_k)$, $M(\q;B)\le e^\varepsilon M(\q';B)+\delta$. Averaging over the shared data-independent randomness yields 
    \begin{align*}
        \Pr[\textsc{Sampler}(S) \in B] 
        &= \E_{\Pi, (U_i)_{i\in[k]}}[M(\q;B)] \\
        &\leq e^\varepsilon \E_{\Pi, (U_i)_{i\in[k]}}[M(\q';B)] + \delta\\
        &= e^\varepsilon \Pr[\textsc{Sampler}(S') \in B] + \delta,
    \end{align*}
    and similarly for the reverse direction. Therefore, \textsc{Sampler} is $(\varepsilon,\delta)$-DP.
\end{proof}

\begin{proposition}[Utility of \textsc{Sampler}]\label{prop:PCS-utility}
Fix a base learner $L$ for $\Q$. Let $\varepsilon >0$, $\beta,\delta\in (0,1)$, $\ell\in \mathbb N^+$ be parameters. Set $\eta = \frac{\beta}{64\ell}$, $\delta_0=\min\{\delta,\beta/4\}$, and choose $k$ such that $k \geq \max\left\{
\frac{8}{\varepsilon}\log\left(\frac{4}{\beta\delta}\right),
128\log\left(\frac{4}{\beta}\right)
\right\}$. 
Let $b\in \mathbb N^+$ be a sample size such that, for every $p\in\Q$, if $S_i\sim p^b$ and $U_i\sim\mathrm U([0,1])$ independently, $\Pr[\TV(L(S_i;U_i),p) \ge \eta] \le 1 /16$.

Set $n=kb$. Then, for every $p\in\Q\subseteq \F$ and $S\sim p^n$, 
$\TV(\mathcal{L}(\textsc{Sampler}(S)),p^\ell) \leq 3\beta / 4$. Thus, \textsc{Sampler} is an $(\ell,3\beta/4)$-approximate sampler
for $\Q$ in the sense of Definition~\ref{def:sampler}.
\end{proposition}

\begin{proof}
    Fix $p\in\Q$. We first show that, with high probability, at least $7k/8$ of the learner outputs are $\eta$-close to $p$. Note that $(X_{\pi(1)},\ldots,X_{\pi(n)}) \sim p^n$, for any fixed permutation $\pi$, since $S \sim p^n$. Hence, conditional on $\Pi = \pi$, the blocks $S_1,...,S_k$, where $S_i=(X_{\pi((i-1)b+1)},...,X_{\pi(ib)})$, are independent and $S_i \sim p^b$ for every $i.$ Moreover, the learner seeds $U_1,...,U_k$ remain independent conditional on $\Pi = \pi$, and are independent of the data $S$. Therefore, conditional on $\Pi = \pi$, $(S_i, U_i),...,(S_k,U_k)$ are independent, and thus, the learner outputs $q_i=L(S_i;U_i)$ are independent. 

    Now define indicators $I_i=1\{\TV(q_i,p) > \eta\}$ for each $i\in[k]$. By the accuracy guarantee of the learner, $\Pr[I_i =1| \Pi=\pi] \le 1/16$ for every $i$, and thus, $I_1,...,I_k$ are independent Bernoulli with means at most $1/16$, conditional on $\Pi=\pi$. By Hoeffding's inequality, $\Pr[\sum_{i=1}^k I_i \ge k/8| \Pi=\pi] \le \exp(-2k (1/8-1/16)^2)=e^{-k/128} \le \beta/4$. Since this holds uniformly for every realization $\pi$, so taking an expectation over $\Pi$ yields $\Pr[\sum_{i=1}^k I_i \ge k/8] \le \beta/4$, i.e. $\Pr[E^c] \le \beta / 4$, where $E = \{|\{i\in[k]:\TV(q_i,p)\le\eta\}|\ge \frac{7k}{8}\}$.

    By Corollary~\ref{cor:instantiated-coupled-sampler-utility}, every realized $\q$ that satisfies $E$, satisfies $\TV(\mathcal{L}(\textsc{Aggregate}(\q)) , p^\ell ] \leq \beta / 2$. Note that $\mathcal{L}(\textsc{Sampler}(S)|E) = \E[\mathcal{L}(\textsc{Aggregate}(\q)| E]$, so $\TV(\mathcal{L}(\textsc{Sampler}(S)|E), p^\ell) \le \beta/2$. Thus, 
    \begin{align*}
        \TV(\mathcal{L}(\textsc{Sampler}(S)),p^\ell) &\leq \Pr[E] \TV(\mathcal{L}(\textsc{Sampler}(S)|E), p^\ell) \\
        &+ \Pr[E^c] \TV(\mathcal{L}(\textsc{Sampler}(S)|E^c), p^\ell) \\
        &\le \beta /2 + \Pr[E^c] \le \beta / 2 + \beta / 4 = 3 \beta / 4.
    \end{align*}
    Finally, by the coupling characterization of total variation distance, there exists a coupling $(\tilde Z,\tilde Z_*)$ such that $\tilde Z \overset{\text{d}}{=} \textsc{Sampler}(S)$, $\tilde Z_* \overset{\text{d}}{=} p^\ell$, and $\Pr[\tilde Z \ne \tilde Z_* ] \leq 3 \beta / 4$. Thus, \textsc{Sampler} is an $(\ell,3\beta/4)$-correlated sampler for $\Q$.
\end{proof}

Combining the dataset-level privacy guarantee of Proposition~\ref{prop:PCS-privacy} with the utility guarantee of Proposition~\ref{prop:PCS-utility} yields the formal guarantee for the end-to-end pipeline \textsc{Sampler}.

\begin{corollary}\label{cor:instantiated-private-correlated-sampler-utility}
Under the same setting as \ref{prop:PCS-utility}, $\textsc{Sampler}$ is $(\varepsilon,\delta_0)$-DP $(\ell,3\beta/4)$-approximate sampler
for $\Q$ in the sense of Definition~\ref{def:sampler}.
\end{corollary}

\section{Agnostic private distribution learning}\label{sec:agnostic-private}
In this section, we extend our main result beyond the realizable setting. We show that if a class of distributions is agnostic learnable, then it is also agnostic private learnable. The following is the definition of agnostic learnability.

\begin{definition}[$C$-agnostic learnability]
\label{def:agnostic-learnability}
Let $C\ge1$ and let
$\Q\subseteq\Delta(\mathcal X)$ be nonempty.
We say that $\Q$ is \emph{$C$-agnostically learnable
with sample complexity $m:(0,1)^2\to\mathbb N_{\ge1}$}
if there exists a randomized learner
$L:\mathcal X^*\to\Delta(\mathcal X)$ such that,
for every $\alpha,\beta\in(0,1)$ and
$p\in\Delta(\mathcal X)$, given
a sample set $S\sim p^{m(\alpha,\beta)}$,
the learner's output $\hat{p}=L(S)$ satisfies $\TV(\hat{p},p)\le C\operatorname{OPT}(p)+\alpha$ with probability at least $\ge1-\beta$, where the probability is over the samples and the learner's internal randomness, and $\operatorname{OPT}(p)\coloneqq\inf_{q\in\Q}\TV(p,q)$,
\end{definition}

\begin{theorem}[General agnostic private reduction]
\label{thm:agnostic-private-general}
Let $\mathcal X$ be standard Borel and $C\ge1$.
If $\Q\subseteq\Delta(\mathcal X)$ is
$C$-agnostically learnable with sample complexity
$m(\alpha,\beta)$, then, for every
$\varepsilon,\delta\in(0,1)$, it is also
$(8C^2+9C)$-agnostically learnable under
$(\varepsilon,\delta)$-differential privacy,
with sample complexity
$n_0(\alpha,\beta,\varepsilon,\delta)=O\left(
m\left(\frac{\alpha}{16(C+3)},\frac{1}{16}\right)
m\left(\frac{2\alpha}{C+3},\beta\right)
\frac{1}{\varepsilon}
\log\left(\frac{4(C+3)}{\alpha\delta}\right)
\right)$
for every $\alpha,\beta\in(0,1)$.
The reduction preserves properness.
\end{theorem}
\begin{proof}
Let $L$ be a $C$-agnostic learner witnessing
Definition~\ref{def:agnostic-learnability}
with sample complexity $m$.
Set $a=\frac{\alpha}{C+3}$,
$b=m\left(\frac{a}{16},\frac{1}{16}\right)$,
$T=m(2a,\beta)$,
$\delta_0=\min\left\{\delta,\frac{a}{4}\right\}$,
and
$k=\left\lceil\max\left\{
\frac{16}{\varepsilon}
\log\left(\frac{1}{\delta_0}\right),
128\log\left(\frac{4}{a}\right)
\right\}\right\rceil$.
Let $n_0=Tkb$.

We use Algorithm~\ref{alg:agnostic-private-learner}
with the parameters chosen above.
Each call to \textsc{Sampler} invokes
Algorithm~\ref{alg:private-correlated-sampler}
on a batch $S_j\in\mathcal X^{kb}$,
using $k$ blocks of size $b$ and base learner
$L_{\mathrm{base}}$, and returns one private sample
$Z_j\in\mathcal X$.
We instantiate \textsc{Sampler} with
ambient family $\F=\Delta(\mathcal X)$,
the coupling implementation from
Algorithm~\ref{alg:coupled-II},
and privacy parameters $(\varepsilon,\delta_0)$.
All invocations use fresh randomness, independent
of the outer permutation and of one another.
The final invocation of $L$ receives
$(Z_1,\ldots,Z_T)$ and returns $\widehat q$
using independent randomness.

For inputs larger than $n_0$, we use only the first $n_0$ observations. For inputs smaller than $n_0$, we return a fixed element of $\Q$.

\textbf{Privacy.}
By Proposition~\ref{prop:PCS-privacy}, each invocation
of \textsc{Sampler} is $(\varepsilon,\delta_0)$-DP.
Conditioned on the data-independent outer permutation,
the invocations use disjoint sets of records.
Parallel composition
(Lemma~\ref{thm:parallel-composition}) therefore implies
that $(Z_1,\ldots,Z_T)$ is $(\varepsilon,\delta_0)$-DP.
Averaging over the permutation preserves this guarantee.
The final estimate is post-processing, so by Proposition~\ref{prop:post-processing} the learner
is $(\varepsilon,\delta)$-DP because $\delta_0\le\delta$.
The fixed output on insufficient data is also private.
If the assumed learnability is proper, choose $L$
proper; the final estimate and the fixed output then
both belong to $\Q$.

\textbf{Utility.}
Fix $p\in\Delta(\mathcal X)$, write
$\operatorname{OPT}\coloneqq\inf_{q\in\Q}\TV(p,q)$,
and consider one invocation of \textsc{Sampler}
on $S_0\sim p^{kb}$.
Let $q_1,\ldots,q_k$ be its base learner outputs and
set $\eta=C\operatorname{OPT}+\frac{a}{16}$.
By the agnostic guarantee of $L$ and the choice of $b$,
$\Pr\left[\TV(q_i,p)\le\eta\right]\ge\frac{15}{16}$
for every $i\in[k]$.

Let
$J=\left\{i\in[k]:\TV(q_i,p)\le\eta\right\}$
and $E=\left\{|J|\ge\frac{7k}{8}\right\}$.
Independence of the blocks and learner randomness
implies that the events $\{\TV(q_i,p)>\eta\}$
are independent, each with probability at most
$\frac{1}{16}$.
Thus, Hoeffding's inequality gives $\Pr[E^c]
    \le\exp\left(
        -2k\left(\frac{1}{8}-\frac{1}{16}\right)^2
    \right)
    =e^{-k/128}
    \le\frac{a}{4}$.

By Fact~\ref{fact:mode}, the mode mechanism instantiated
with $(\varepsilon,\delta_0)$ has
$\tau=\frac{4}{\varepsilon}
\log\left(\frac{1}{\delta_0}\right)$
and $\beta_{\mathsf{SBM}}=\delta_0$.
On $E$, taking $c=\frac{1}{4}$, our choice of $k$
ensures $2(1-c)|J|-k
    \ge\frac{5k}{16}
    \ge\frac{k}{4}
    \ge\tau$.

For each tuple $\mathbf q$, let
$\mu_{\mathbf q}\coloneqq
\mathcal L(\mathsf{Aggregate}(\mathbf q))$.
By Proposition~\ref{prop:coupled-sampler-utility}
with $\ell=1$, on $E$ we have
$\TV(\mu_{\mathbf q},p)\le8\eta+\delta_0$.
This holds with probability at least
$1-e^{-k/128}\ge1-\frac{a}{4}$ over the block data
and base learner randomness.

Let $R_p\coloneqq\mathcal L(\mathsf{Sampler}(S_0))$,
where the law includes both the randomness of
$S_0\sim p^{kb}$ and the sampler's internal randomness.
By Algorithm~\ref{alg:private-correlated-sampler}
and the law of total probability, for every measurable
$B\subseteq\mathcal X$,
$R_p(B)=\mathbb E_{\mathbf q}[\mu_{\mathbf q}(B)]$.
Consequently,
\[
\begin{aligned}
    \TV(R_p,p)
    &=\sup_B\left|
        \mathbb E_{\mathbf q}
        [\mu_{\mathbf q}(B)-p(B)]
      \right|\\
    &\le\mathbb E_{\mathbf q}\left[
        \sup_B|\mu_{\mathbf q}(B)-p(B)|
      \right]\\
    &=\mathbb E_{\mathbf q}
      [\TV(\mu_{\mathbf q},p)]\\
    &\le(8\eta+\delta_0)\Pr[E]+\Pr[E^c]\\
    &\le8\eta+\delta_0+e^{-k/128}\\
    &\le8C\operatorname{OPT}
        +\frac{a}{2}+\frac{a}{4}+\frac{a}{4}\\
    &=8C\operatorname{OPT}+a.
\end{aligned}
\]

The data-independent permutation preserves the joint
distribution $p^{Tkb}$, so the batches
$S_1,\ldots,S_T$ are independent, each with distribution
$p^{kb}$.
Each invocation applies the same \textsc{Sampler},
with the same parameters and independent randomness.
Hence, for every $j\in[T]$,
$\mathcal L(Z_j)=R_p$.
Since each $Z_j$ depends only on $S_j$ and the
randomness of its invocation, the outputs are independent.
Therefore, $(Z_1,\ldots,Z_T)\sim R_p^T$.

Applying the $C$-agnostic guarantee of $L$ with
$T=m(2a,\beta)$ samples from $R_p$ gives
\[
    \TV(\widehat q,R_p)
    \le C\inf_{q\in\Q}\TV(q,R_p)+2a
\]
with probability at least $1-\beta$.
By the triangle inequality,
$\inf_{q\in\Q}\TV(q,R_p)
\le\operatorname{OPT}+\TV(p,R_p)$.
Thus, on the same event,
\[
\begin{aligned}
    \TV(\widehat q,p)
    &\le\TV(\widehat q,R_p)+\TV(R_p,p)\\
    &\le C\operatorname{OPT}
        +(C+1)\TV(R_p,p)+2a\\
    &\le C\operatorname{OPT}
        +(C+1)(8C\operatorname{OPT}+a)+2a\\
    &=(8C^2+9C)\operatorname{OPT}+(C+3)a\\
    &=(8C^2+9C)\operatorname{OPT}+\alpha.
\end{aligned}
\]

Finally,
$\log\left(\frac{1}{\delta_0}\right)
=\max\left\{
\log\left(\frac{1}{\delta}\right),
\log\left(\frac{4}{a}\right)
\right\}
\le\log\left(\frac{4}{a\delta}\right)$.
Since $\varepsilon\in(0,1)$,
$k=O\left(
\frac{1}{\varepsilon}
\log\left(\frac{4}{a\delta}\right)
\right)$.
Therefore,
\[
\begin{aligned}
    n_0=Tkb
    &=O\left(
        m\left(\frac{a}{16},\frac{1}{16}\right)
        m(2a,\beta)
        \frac{1}{\varepsilon}
        \log\left(\frac{4}{a\delta}\right)
      \right)\\
    &=O\left(
        m\left(\frac{\alpha}{16(C+3)},\frac{1}{16}\right)
        m\left(\frac{2\alpha}{C+3},\beta\right)
        \frac{1}{\varepsilon}
        \log\left(\frac{4(C+3)}{\alpha\delta}\right)
      \right),
\end{aligned}
\]
which proves the stated sample bound.
\end{proof}

Next, we present a generic agnostic private learner based on the non-private Minimum distance estimator of \cite{DevroyeL01}. This automatically gives an agnostic private learner for every countable
class whose Yatracos class has finite VC dimension.

\begin{definition}[Yatracos class]
    Let $\mathcal{Q} \subseteq \Delta(\mathcal{X})$ be a non-empty class of distributions. For every ordered pair $(q,q')\in \mathcal{Q}$, fix a measurable set $A_{q,q'}\subseteq \mathcal{X}$ satisfying $q(A_{q,q'})-q'(A_{q,q'})=\operatorname{TV}(q,q')$\footnote{When the distributions in $\mathcal Q$ have densities $f_q$ with respect to a
common dominating measure, we take
$A_{q,q'}=\{x\in\mathcal X:f_q(x)>f_{q'}(x)\}$, recovering the usual
definition.}.
    The corresponding Yatracos class is defined as:
    \[
    \operatorname{Yat}(\mathcal Q)
    \coloneqq\{A_{q,q'}:q,q'\in\mathcal Q\}.
    \]
\end{definition}

\begin{theorem}[Minimum-distance estimator, \cite{DevroyeL01}]
\label{thm:agnostic-yatracos-mde}
Let $C>0$ be a sufficiently large universal constant and
$\alpha,\beta\in(0,1)$. Let $\mathcal Q\subseteq\Delta(\mathcal X)$
be a nonempty countable\footnote{The countability assumption is used (as a sufficient condition) only to avoid measurability issues. Indeed, if $\mathcal Q$ is countable, then $\operatorname{Yat}(\mathcal Q)$ is also countable, so the minimum-distance objective $D(S,q)\coloneqq\sup_{A\in\operatorname{Yat}(\mathcal Q)}|P_S(A)-q(A)|$ is measurable in $S$ for every $q\in\mathcal Q$., where $P_S$ is the empirical measure. Fixing an enumeration of $\mathcal Q$, we may therefore return the first $q$ satisfying $D(S,q)\le\inf_{q'\in\mathcal Q}D(S,q')+\alpha/2$. This selection is measurable, and the usual MDE analysis gives the stated guarantee after adjusting the universal constant.} class of distributions with
$d=\operatorname{VC}(\operatorname{Yat}(\mathcal Q))<\infty$.
There exists an algorithm, called $\mathrm{MDE}$, that, for every
$p\in\Delta(\mathcal X)$, given a dataset $S\sim p^n$ of size
$n\ge C(d+\log(2/\beta))/\alpha^2$, returns
$\widehat q\in\mathcal Q$ satisfying
\[
    \Pr\!\left[
        \operatorname{TV}(\widehat q,p)
        \le 3\operatorname{OPT}+\alpha
    \right]\ge 1-\beta,
\]
where $\operatorname{OPT}\coloneqq\inf_{q\in\mathcal Q}\operatorname{TV}(p,q)$.
\end{theorem}

\begin{algorithm}[t]
\caption{\textsc{AgnosticPrivateLearner}}
\label{alg:agnostic-private-learner}
\begin{algorithmic}[1]
\Require Dataset $S=(X_1,\ldots,X_n)$ with $n=Tkb$,
a non-private $C$-agnostic learner $L$,
and parameters $a,b,k,T,\beta,\varepsilon,\delta_0$.
\Ensure An estimate $\widehat q\in\Delta(\mathcal X)$.
\State Define
$L_{\mathrm{base}}(D)\coloneqq
L\left(D;\frac{a}{16},\frac{1}{16}\right)$.
\State Draw a data-independent permutation
$\Pi\sim\mathrm{U}(\mathcal S_n)$.
\For{each $j\in[T]$}
    \State $S_j\gets
    (X_{\Pi((j-1)kb+1)},\ldots,X_{\Pi(jkb)})$.
    \State $Z_j\gets\textsc{Sampler}
    (S_j;\ell=1,k,b,L_{\mathrm{base}},
    \varepsilon,\delta_0)$.
\EndFor
\State $\widehat q\gets
L\left((Z_1,\ldots,Z_T);2a,\beta\right)$.
\State \Return $\widehat q$.
\end{algorithmic}
\end{algorithm}

The following corollary gives a generic agnostic private learner for any countable class with finite $\operatorname{VC}(\operatorname{Yat}(\mathcal Q))$.
\begin{corollary}[Agnostic private learning Yatracos classes]
\label{thm:agnostic-private-yatracos}
Let $C>0$ be a sufficiently large universal constant and
$\alpha,\beta,\varepsilon,\delta\in(0,1)$.
Let $\mathcal Q\subseteq\Delta(\mathcal X)$ be a nonempty countable class of
distributions with $d=\operatorname{VC}(\operatorname{Yat}(\mathcal Q))<\infty$.
There exists an $(\varepsilon,\delta)$-differentially private algorithm
that, for every $p\in\Delta(\mathcal X)$, given a dataset $S\sim p^n$
of size
$n\ge C\frac{(d+1)(d+\log(2/\beta))}{\varepsilon\alpha^4}
    \log(\frac{2}{\alpha\delta})$,
returns $\widehat q\in\mathcal Q$ satisfying
\[
    \Pr\!\left[
        \operatorname{TV}(\widehat q,p)
        \le 99\operatorname{OPT}+\alpha
    \right]\ge 1-\beta,
\]
where $\operatorname{OPT}\coloneqq\inf_{q\in\mathcal Q}\operatorname{TV}(p,q)$.
\end{corollary}
\begin{proof}
In Algorithm~\ref{alg:agnostic-private-learner}, take
$L(D;t,\rho)\coloneqq\mathrm{MDE}(D;t,\rho)$.
By Theorem~\ref{thm:agnostic-yatracos-mde}, this is a
proper $3$-agnostic learner for $\Q$ with sample complexity
$m(t,\rho)=\left\lceil
C_0\frac{d+\log\left(2/\rho\right)}{t^2}
\right\rceil$,
where $C_0$ is a universal constant.

Theorem~\ref{thm:agnostic-private-general} therefore
makes Algorithm~\ref{alg:agnostic-private-learner}
$(\varepsilon,\delta)$-DP and guarantees
$\widehat q\in\Q$ with
$\TV(\widehat q,p)\le99\operatorname{OPT}+\alpha$
with probability at least $1-\beta$.

Its sample bound is
$n_0=O\left(
m\left(\frac{\alpha}{96},\frac{1}{16}\right)
m\left(\frac{\alpha}{3},\beta\right)
\frac{1}{\varepsilon}
\log\left(\frac{24}{\alpha\delta}\right)
\right)$.
Since
$m\left(\frac{\alpha}{96},\frac{1}{16}\right)
=O\left(\frac{d+1}{\alpha^2}\right)$
and
$m\left(\frac{\alpha}{3},\beta\right)
=O\left(
\frac{d+\log\left(2/\beta\right)}{\alpha^2}
\right)$,
this becomes
$n_0=O\left(
\frac{(d+1)\left(d+\log\left(2/\beta\right)\right)}
{\varepsilon\alpha^4}
\log\left(\frac{2}{\alpha\delta}\right)
\right)$,
which proves the claimed guarantee.
\end{proof}

For potentially infinite classes, Bun et al.~\cite[Theorem~4]{BunKSW19}
give a private hypothesis selection guarantee depending on
$d=\operatorname{VC}(\operatorname{Yat}(\mathcal Q))$ and a locally
small cover. An $\alpha$-cover $\mathcal H\subseteq\mathcal Q$
contains an element within TV distance $\alpha$ of every
$q\in\mathcal Q$, and is $(k,r)$-locally small if
$\sup_{h\in\mathcal H}|\{h'\in\mathcal H:\TV(h,h')\le r\}|\le k$.
Applied to a $(k,8\alpha)$-locally small $\alpha$-cover, their
$(\varepsilon,\delta)$-private algorithm uses
$O\!\left(\frac{d+\log(1/\beta)}{\alpha^2}
+\frac{\log(k/\beta)+\log(1/\delta)}{\alpha\varepsilon}\right)$
samples and achieves the $7$-agnostic guarantee
$\TV(\widehat q,p)\le7\alpha$ with probability at least $1-\beta$,
provided $\TV(p,h)\le\alpha$ for some $h\in\mathcal H$.
This guarantee requires a supplied approximation bound.
For countable classes, Corollary~\ref{thm:agnostic-private-yatracos} removes both this
requirement and the locally small cover assumption: finite $d$
alone suffices for a $99$-agnostic private learner, guaranteeing
$\TV(\widehat q,p)\le99\operatorname{OPT}+\alpha$ with probability
at least $1-\beta$ for every $p$.
This gives a qualitative improvement in the assumptions, at the
cost of a larger sample complexity and approximation constant.\\

\section{Omitted proof from Section~\ref{sec:boosted-learning}}\label{sec:proof:boosted-learner}
Here, we formally prove the guarantee for the boosted learner
introduced in Section~\ref{sec:boosted-learning}.
Its core idea is confidence boosting: amplifying a learner's
constant success probability to $1-\beta$ through independent
repetition and robust aggregation.

\begin{proof}[Proof of Theorem~\ref{thm:robust-learner}] 
Set $\alpha_0=\frac{\alpha}{3}$, $\beta_0=\frac{1}{8}$, $\delta_0=\min\left\{\delta,\beta_0/4\right\}$, and $T=\left\lceil128\ln\left(1/\beta\right)\right\rceil$. Fixing $\beta_0$ keeps the block size $c$ independent of the
target failure probability $\beta$. We amplify confidence
by increasing the number of independent runs $T$. By Theorem~\ref{thm:main}, each private learner is $(\varepsilon,\delta_0)$-DP. Since the runs use disjoint blocks and independent randomness, parallel composition (Lemma~\ref{thm:parallel-composition}) gives the same privacy guarantee for the vector of estimates. Selecting $\hat p_j$ uses only these estimates, so post-processing preserves privacy (Proposition~\ref{prop:post-processing}). 
Since $\delta_0\le\delta$, the overall procedure is $(\varepsilon,\delta)$-DP.  

For utility, Theorem~\ref{thm:main} provides a block size $c=O\left( m\left( \frac{\beta_0} {64m\left(\alpha_0,\beta_0/4\right)}, \frac{1}{16} \right) \frac{1}{\varepsilon} \log\left(\frac{4}{\beta_0\delta_0}\right) \right)$ such that each estimate satisfies $\TV(\hat p_t,p)<\alpha_0$ with probability at least $1-\beta_0$. Fix $p\in\Q$ and let $\mathcal T=\left\{t\in[T]: \TV(\hat p_t,p)<\alpha_0\right\}$. The estimates are independent because the input is i.i.d., the blocks are disjoint, and the runs use independent randomness. Hoeffding's inequality therefore gives $\Pr\left[|\mathcal T|\le\frac{T}{2}\right] \le\exp\left( -2T\left(\frac{1}{2}-\beta_0\right)^2 \right)\le\beta$.  On the event $|\mathcal T|>T/2$, any two estimates indexed by $\mathcal T$ are within $2\alpha_0$ of one another. Thus, $N_t\ge|\mathcal T|$ for every $t\in\mathcal T$, and the maximizing index $j$ satisfies $N_j\ge|\mathcal T|>T/2$. The estimates counted by $N_j$ must include some $i\in\mathcal T$. Consequently, $\TV(\hat p_j,p) \le\TV(\hat p_j,\hat p_i)+\TV(\hat p_i,p) <3\alpha_0=\alpha$. This proves the accuracy guarantee.  Finally, $n=Tc$. Since $\beta_0=1/8$ and $\delta_0=\min\left\{\delta,\frac{1}{32}\right\}$, we have $\log\left(\frac{4}{\beta_0\delta_0}\right) =O\left(\log\left(\frac{4}{\delta}\right)\right)$ and $T=O\left(\log\left(\frac{4}{\beta}\right)\right)$. Substituting these choices gives the stated bound. 

Finally, if $m(\alpha,\beta)=\operatorname{poly}\left( \frac{1}{\alpha},\log\left(\frac{1}{\beta}\right)\right)$, this bound is $\operatorname{poly}\left( \frac{1}{\alpha},\log\left(\frac{1}{\beta}\right), \frac{1}{\varepsilon},\log\left(\frac{1}{\delta}\right) \right)$. 
\end{proof}

\end{document}